\pdfoutput=1
\documentclass[letterpaper, 10 pt, conference]{ieeeconf} 
\IEEEoverridecommandlockouts                              

\makeatletter
\let\NAT@parse\undefined
\makeatother

\usepackage{amssymb,amsmath}

\usepackage[sort&compress,numbers,sectionbib]{natbib}

\usepackage{cite} 
\usepackage{booktabs}
\usepackage[font=small]{caption}
\usepackage{subcaption}
\usepackage{float}
\usepackage{graphicx}

\usepackage{verbatimbox}
\usepackage{multirow}
\usepackage{balance}
\usepackage{algorithm}
\usepackage{algorithmic}
\usepackage{comment}

\usepackage{dsfont}

\newcommand{\figref}[1]{Fig.~\ref{#1}}
\newcommand{\tabref}[1]{Table~\ref{#1}}

\usepackage[hyphens]{url}

\usepackage{hyperref}
\hypersetup{bookmarksopen,bookmarksnumbered,
pdfpagemode=UseOutlines,
colorlinks=true,
linkcolor=teal,
anchorcolor=teal,
citecolor=teal,
filecolor=teal,
menucolor=teal,
urlcolor=teal,
breaklinks=true
}
\newtheorem{theorem}{Theorem}[section]

\usepackage[dvipsnames]{xcolor}
\definecolor{darkblue}{RGB}{0.15,0.15,0.55}
\definecolor{lightgrey}{RGB}{0.75,0.75,0.75}

\definecolor{myred}{RGB}{215,48,39}
\definecolor{myblue}{RGB}{69,117,180}
\definecolor{myorange}{RGB}{252,141,89}
\definecolor{mylightblue}{RGB}{145,191,219}

\definecolor{MYlightblue}{RGB}{217,95,2} 
\definecolor{MYdarkblue}{RGB}{117,112,179} 
\definecolor{MYgreen}{RGB}{27,158,119}

\usepackage{color,soul}
\usepackage{svg}

\newcommand{\xxnote}[3]{}
\ifx\hidenotes\undefined
  \usepackage{color}
  \renewcommand{\xxnote}[3]{\color{#2}{#1: #3}}
\fi

\usepackage[normalem]{ulem}
\useunder{\uline}{\ul}{}

\usepackage{bm}


\title{\LARGE \bf ReloPush-BOSS: Optimization-guided Nonmonotone Rearrangement Planning for a Car-like Robot Pusher}
\author{Jeeho Ahn and Christoforos Mavrogiannis\thanks{The authors are with the Department of Robotics, University of Michigan, Ann Arbor, USA. Email: \tt\small \{jeeho, cmavro\}@umich.edu}}

\let\oldsubsection\subsection
\renewcommand{\subsection}[1]{%
    \vspace{-3pt}
    \oldsubsection{#1}%
    \vspace{-3pt}
}

\begin{document}

\maketitle
\thispagestyle{empty}
\pagestyle{empty}

\begin{abstract}

We focus on multi-object rearrangement planning in densely cluttered environments using a car-like robot pusher. The combination of kinematic, geometric and physics constraints underlying this domain results in challenging \emph{nonmonotone} problem instances which demand breaking each manipulation action into multiple parts to achieve a desired object rearrangement. Prior work tackles such instances by planning \emph{prerelocations}, temporary object displacements that enable constraint satisfaction, but deciding where to \emph{prerelocate} remains difficult due to local minima leading to infeasible or high-cost paths. Our key insight is that these minima can be avoided by steering a prerelocation optimization toward low-cost regions informed by Dubins path classification.
These optimized prerelocations are integrated into an object traversability graph that encodes kinematic, geometric, and pushing constraints. Searching this graph in a depth-first fashion results in efficient, feasible rearrangement sequences. Across a series of densely cluttered scenarios with up to 13 objects, our framework, \emph{ReloPush-BOSS}, exhibits consistently highest success rates and shortest pushing paths compared to state-of-the-art baselines. Hardware experiments on a 1/10 car-like pusher demonstrate the robustness of our approach. Code and footage from our experiments can be found at: \url{https://fluentrobotics.com/relopushboss}.

\end{abstract}


\section{Introduction}\label{sec:introduction}




Rearrangement planning has revolutionized fulfillment by enabling rapid, large-scale transportation of goods by mobile robots in massive warehouses~\citep{dandrea2012kiva}. Transportation tasks in these domains are often facilitated by extensive workspace and systems engineering: robots often move along rectilinear grids in well-structured spaces; they manipulate a closed set of objects of regular shapes and sizes; they apply secure closure grasps through grippers of high design complexity and cost. This level of structure is not possible in other important domains like construction, waste management, and smaller/medium warehouses due to high cost and clutter density. In those environments, the ability to handle a wide range of objects is essential, and so is the need to address kinematic and geometric constraints. These requirements give rise to complex problem instances that motivate the investigation of new paradigms for rearrangement planning.

A practical strategy for tackling rearrangement tasks in unstructured domains is pushing~\citep{lynch1996stable,mason1986mechanics}: exploiting physics, rather than grasping, pushing allows for the reconfiguration of large, heavy, and irregular objects without requiring specialized grippers; a simple pushing surface is sufficient. However, this paradigm introduces the challenge of integrating physics into rearrangement planning. A practical way of doing so is via constraining interactions between the robot and pushed objects to be quasistatic. This approach further constraints robot motion: quasistatic pushing using a mobile robot reduces to the Dubins car problem~\citep{lynch1996stable}. When the robot is constrained, its maneuverability is also limited even when \emph{not} pushing. In densely cluttered spaces, this combination of geometric, kinematic, and physics constraints renders many practical planning problems infeasible.


\begin{figure}[t]
\begin{subfigure}{0.485\linewidth}
\includegraphics[width=\linewidth]{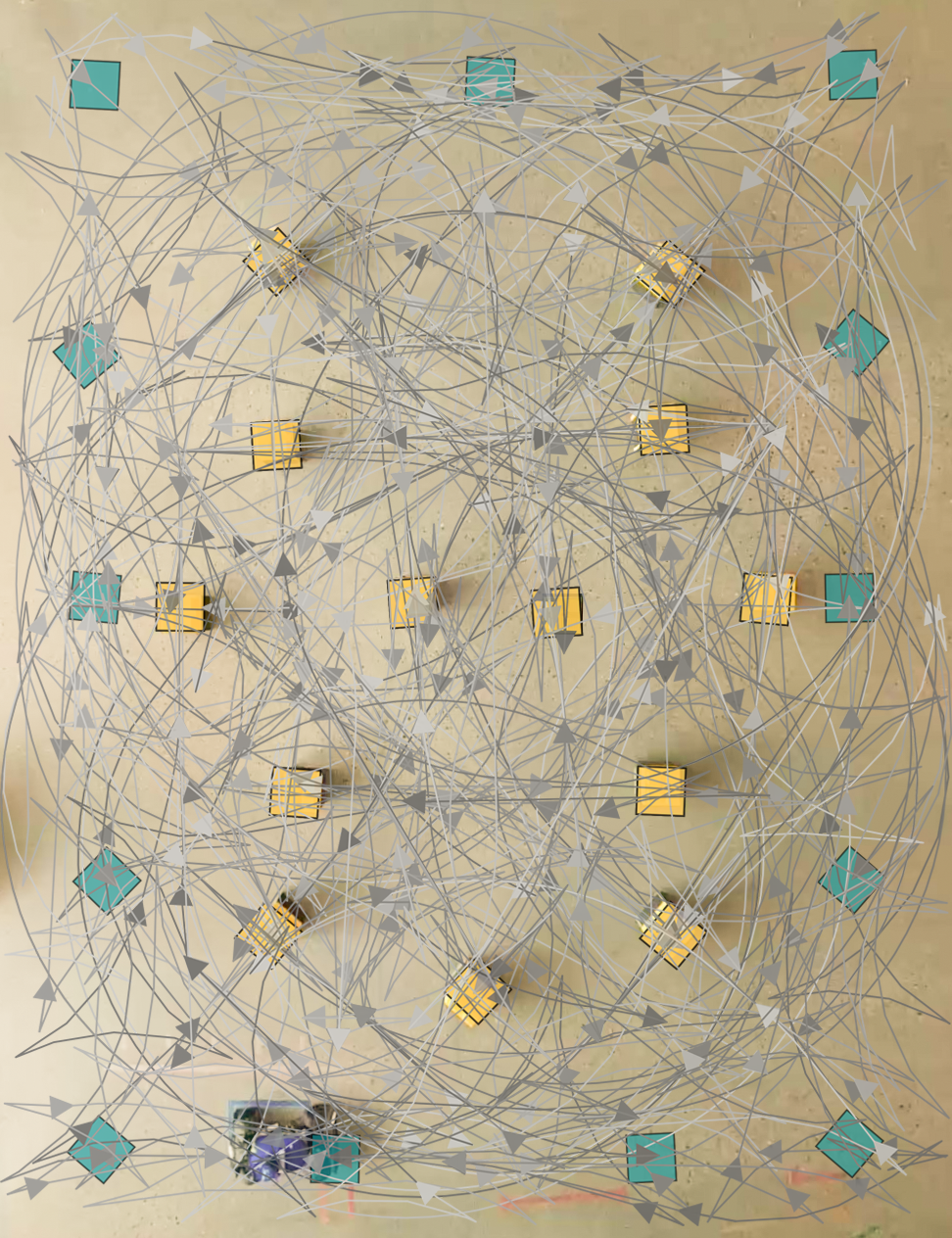}
\caption{Initial (yellow) and goal (teal) object poses, connected through a \emph{Push-Traversability} graph.}
\label{fig:relopush_before}
\end{subfigure}
\begin{subfigure}{0.485\linewidth}
\includegraphics[width=\linewidth]{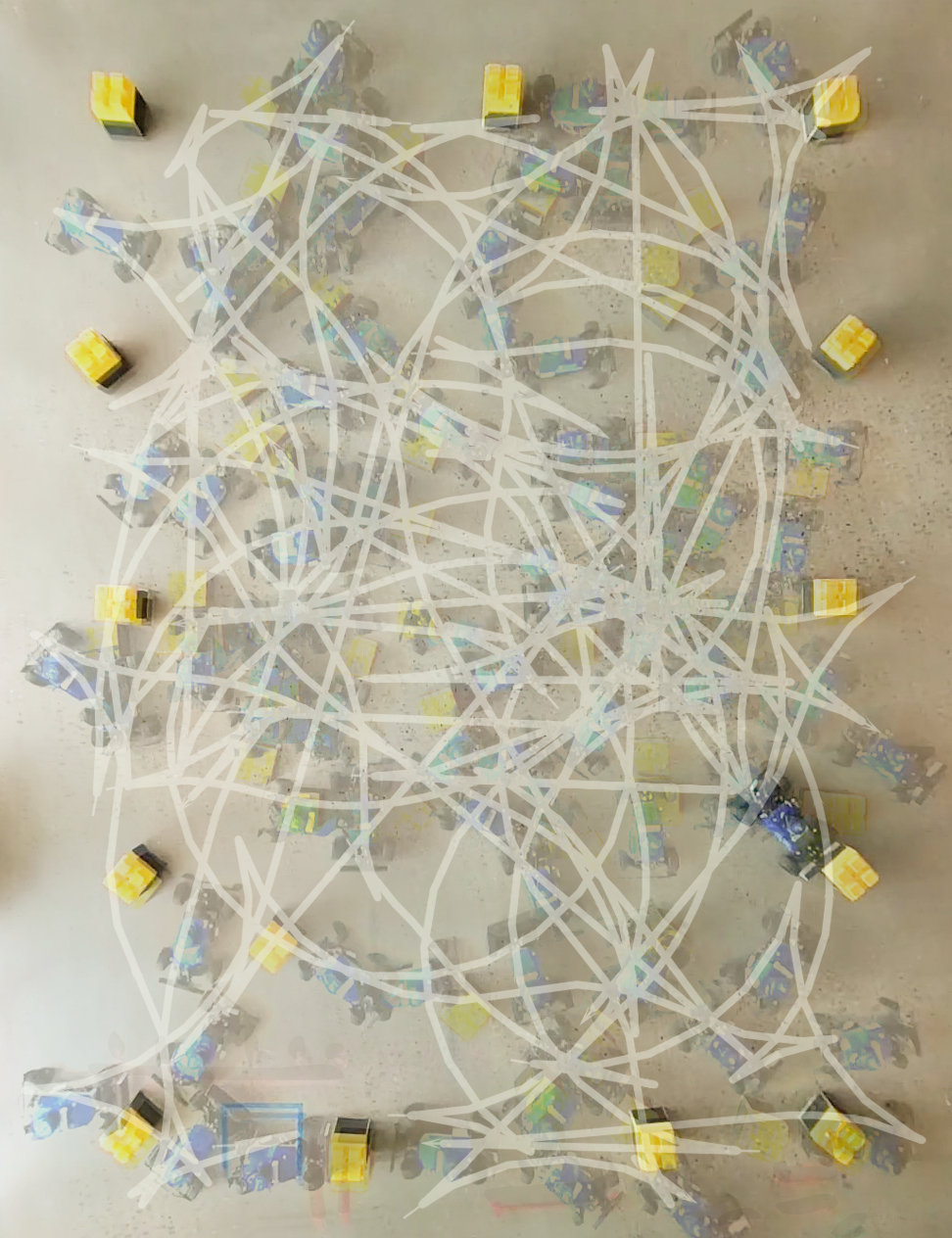}
\caption{Rearrangement path, extracted by planning on the graph, and robot motion traces from an experiment.}
\label{fig:relopush_after}
\end{subfigure}
  \caption{We introduce \textit{ReloPush-BOSS}, an optimization-based planning framework for nonprehensile multiobject rearrangement planning using a car-like robot pusher. By capturing kinematic, geometric, and physics constraints into a unified object traversability graph (\subref{fig:relopush_before}), our framework is capable of handling nonmonotone instances in densely cluttered workspaces with up to 13 objects (\subref{fig:relopush_after}).}
  \label{fig:main}
\end{figure}

In this work, we introduce a planning framework for constrained multiobject rearrangement planning using a car-like robot pusher. Our framework simultaneously addresses geometric (obstacle avoidance, boundary constraints), kinematic (nonholonomic), and physics constraints (quasistatic pushing) to tackle challenging rearrangement tasks (see~\figref{fig:main}). Our approach builds upon our graph-based framework, ReloPush~\citep{ahn2025relopush}, which handles constraints by temporarily prerelocating objects along collision-free Dubins paths leading to object goal poses. In ReloPush, these prerelocations are planned along linear segments along the body frame axes of the pushed object, and thus often lead to high-cost solutions or failures. Here, we introduce \emph{ReloPush-BOSS} (\emph{ReloPush with Backtracking and Optimization with Seeded Start}), a planning framework that combines backtracking and prerelocation optimization informed by Dubins path classification to avoid excessively high-cost local minima. Across a series of simulated and real-world experiments involving the rearrangement of up to 13 objects by a 1/10th-scale robot racecar, we demonstrate the superior scalability and robustness of ReloPush-BOSS against ablated versions as well as a baseline from the literature. 


Our contributions are as follows:
\begin{itemize}
    \item \emph{Optimization-based prerelocation search}: In contrast to our prior work~\citep{ahn2025relopush}, in which prerelocations are sampled from straight-line segments originating from an object's frame, we formulate prerelocation search as a continuous optimization producing lower-cost solutions.
    \item \emph{Avoidance of high-cost local minima}: Exploiting insights from Dubins path classification~\citep{shkel2001classification,dubins1957curves,CHEN2019368,lim2023circling}, we introduce seeded warm-starts in the optimization process to avoid getting trapped into high-cost local minima representing undesired paths.
    \item \emph{Extensive empirical validation}: We validate our framework on challenging constrained planning problems, showing significant performance gains over baselines~\citep{ahn2025relopush,krontiris2015dealing}. We demonstrate the efficacy and robustness of our approach through real-world demonstrations on a 1/10 scale robot racecar~\citep{srinivasa2019mushr}.
\end{itemize}



\begin{table*}[t]
    \setlength{\tabcolsep}{6pt}
    \centering
    \caption{Related works. $^\dagger$ Relocation of obstacles for object retrieval. $^{\dagger\!\dagger}$ Relocation of objects for sorting. $^{\dagger\!\dagger\!\dagger}$ Rearrangement of at most one object, and at most one reconfiguration of the object before transport. $^*$ Orientations are specified for object goal poses. \label{tab:relatedworks}}
    \resizebox{\linewidth}{!}{\begin{tabular}{l c c c c c c c}
        \toprule
        & \multicolumn{2}{c}{Manipulation} & \multicolumn{2}{c}{Hardness} & \multicolumn{3}{c}{Dexterity} \\
        \cmidrule(lr){2-3}\cmidrule(lr){4-5}\cmidrule(lr){6-8}
        & Prehensile & Nonprehensile & Nonmonotone & Orientation Requirement$^*$ & Overhand Grasp & Nonholonomic & DoF Assumed \\
        \midrule
        \citet{stilman2007manipulation}         & \checkmark &             &             & \checkmark     &             &             & High-DoF arm \\
        \citet{dogar2012planning}$^\dagger$     & \checkmark & \checkmark  &             & \checkmark     &             &             & High-DoF arm \\
        \citet{nam2021fast}$^{\dagger}$         & \checkmark &             & \checkmark  &          &             &             & High-DoF arm \\
        \citet{ahn2021integrated}$^\dagger$     & \checkmark &             & \checkmark  &          &             &             & High-DoF arm \\
        \citet{ren2024multi}$^{\dagger}$        & \checkmark &             & \checkmark  &          &             &             & High-DoF arm \\
        \citet{lee2021tree}                     & \checkmark & \checkmark  & \checkmark  &          &             &             & High-DoF arm \\
        \citet{King-RSS-13}$^{\dagger\!\dagger\!\dagger}$ & \checkmark & \checkmark & \checkmark  &  \checkmark      &             &             & High-DoF arm \\
        \citet{krontiris2015dealing}            & \checkmark &             & \checkmark  & \checkmark &             &             & High-DoF arm \\
        \citet{shome2018fast}                   & \checkmark &             &             & \checkmark & \checkmark   &             & High-DoF arm \\
        \citet{ahn2023coordination}$^{\dagger\!\dagger}$ & \checkmark &             & \checkmark  &     & \checkmark   &             & High-DoF arm \\
        \citet{han2018complexity}               & \checkmark &             & \checkmark  & \checkmark     & \checkmark   &             & High-DoF arm \\
        \citet{gao2023effectively}              & \checkmark &             & \checkmark  & \checkmark     & \checkmark   &             & High-DoF arm \\
        \citet{song2019object}                  & \checkmark & \checkmark  & \checkmark  & \checkmark     & \checkmark   &             & High-DoF arm \\
        \citet{tang2023selective}               & \checkmark & \checkmark  & \checkmark  & \checkmark     & \checkmark   &             & High-DoF arm \\
        \citet{huang2019large}$^{\dagger\!\dagger\!\dagger}$ &          & \checkmark  & \checkmark  & \checkmark     &             &             & High-DoF arm \\
        \citet{talia2023pushr}                  &           & \checkmark  &             & \checkmark &             & \checkmark  & Low-DoF (car) \\
        Our prior work, ReloPush~\citep{ahn2025relopush}                 &           & \checkmark  & \checkmark  & \checkmark &             & \checkmark  & Low-DoF (car) \\
        \textbf{This work (ReloPush-BOSS)}                              &           & \checkmark  & \checkmark  & \checkmark &             & \checkmark  & Low-DoF (car) \\
        \bottomrule
    \end{tabular}}
\end{table*}

\section{Related Work}\label{sec:related-work}

There is a long history of work on rearrangement planning. We organize related work in terms of three key categories (see~\tabref{tab:relatedworks}): manipulation strategy; system dexterity; rearrangement hardness. We elaborate on each of them below.



\textbf{Manipulation strategy}. Much of the work on rearrangement planning assumes \emph{prehensile} manipulation via grasping~\citep{stilman2007manipulation,han2018complexity,krontiris2015dealing,gao2023effectively,lee2021tree}, including suction-based pipelines~\citep{han2018complexity}, pick-and-place sequences~\citep{gao2023effectively,lee2021tree}, and grasp-based obstacle relocation for target access~\citep{stilman2007manipulation}. Recent prehensile methods address such sequential challenges via multi-objective planning with selective obstacle relocation to feasible points~\citep{tuncer2025mo} or decomposition into easier pick-and-place subproblems~\citep{levit2024solving}. 
In contrast, \emph{nonprehensile} manipulation~\citep{lynch1996stable,mason1986mechanics,goyal1991planar} leverages physics to handle diverse object sizes/shapes without specialized grippers, making it robust for unstructured domains like cluttered warehouses~\citep{huang2019large,talia2023pushr,bertoncelli2020linear,tang2023unwieldy,song2019object}. For instance, \citet{huang2019large} employ policy rollouts and iterated local search to tackle large-scale planar rearrangement problems using a manipulator that pushes, and \citet{song2020multi} propose a nonprehensile, push-based rearrangement method for object sorting. Some work~\citep{dogar2012planning,song2019object} combines pushing with picking to enable grasping in dense clutter under occlusions, whereas recent work develops policies for robust pushing of diverse objects~\citep{bertoncelli2020linear,tang2023unwieldy}.

\textbf{Dexterity}. For pick-and-place or sorting tasks, high–DoF manipulators are often employed thanks to their increased dexterity which enables complex rearrangements in clutter. In open-top scenes, arms exploit \emph{overhand} access to simplify collision avoidance and sequencing \citep{han2018complexity,gao2023effectively,song2019object,lee2021tree}. High-DoF manipulators are also effective in closed-top settings (shelves/racks/bins) by exploiting side-access grasps to complete challenging rearrangements \citep{nam2021fast,ahn2021integrated,ren2024multi}. However, for large planar workspaces, mobility is prioritized over dexterity. Holonomic bases allow unconstrained planar motion and simplify planning, but many common platforms are nonholonomically constrained (e.g., differential drive, or car-like robots), thus affording limited maneuverability in constrained spaces~\citep{lavalle2006planning,dubins1957curves,reeds1990optimal}.
\citet{ma2019lifelong} incorporate kinematic constraints in a pickup–and–delivery problem under grid-world assumptions whereas \citet{bertoncelli2020linear} address more realistic settings in pushing, and \citet{talia2023pushr} propose a framework for multirobot planar pushing using car-like robots. Closest to our settings, \citet{ahn2025relopush} tackle multiobject rearrangement planning using a car-like robot, but their method is prone to high-cost local minima.


\textbf{Hardness}. Earlier work focused on \emph{monotone} rearrangement tasks where each object is assumed to be moved at most once. The work of~\citet{stilman2007manipulation} handles manipulations of obstacles assuming monotone transitions. \citet{dogar2012planning} present a planning framework for object retrieval formulated under \emph{monotone} assumptions using a manipulator. Similarly, \citet{ahn2022coordination} propose a coordination method of two manipulators in object retrieval, making problems monotone by introducing additional space for removing obstacles from a clutter. \citet{talia2023pushr} address object rearrangement problems using nonholonomic robots, restricted to monotone cases. In contrast, many real-world domains like logistics, construction, and the home often require the rearrangement of multiple objects in dense clutter, thus giving rise to \emph{nonmonotone} problem instances where multiple manipulation actions are required for rearranging any object. These are substantially more complex due to the combinatorial explosion in the search space and interdependent object movements~\citep{han2018complexity,krontiris2015dealing,pan2022algorithms}. This complexity escalates further when constrained by pushing actions and nonholonomic robot kinematics \citep{song2020multi}, and prior work tackles simplified problem instances like the rearrangement of a single object in an obstacle-free workspace~\citep{bertoncelli2020linear,tang2023unwieldy}.

\textbf{This work}. We address labeled rearrangement tasks under nonprehensile, nonholonomic, and geometric constraints, which often result in nonmonotone instances. These instances are substantially more constrained compared to much of the literature, which tends to assume secure, prehensile manipulation (i.e., closure grasps), unconstrained robot motion (e.g., holonomic or overhand grasping), or position-only rearrangement planning. In prior work~\citep{ahn2025relopush}, we introduced \emph{prerelocations} as temporary object displacements that unlock kinematically feasible object transitions along contact normals (pushing axes). This is constraining, as it limits the search space and often excludes more efficient rearrangement routes that could be unlocked by small displacements in other directions. 
Our new method, ReloPush-BOSS addresses this limitation by expanding the search for prerelocations beyond pushing directions, incorporating an optimization strategy that exploits insights from Dubins path classification~\citep{dubins1957curves,CHEN2019368,lim2023circling,shkel2001classification} to identify superior candidates while avoiding getting trapped in high-cost minima.

\section{Problem Statement}\label{sec:statement}

We consider the problem of rearranging a set of $m$ rigid, polygonal objects using a car-like robot pusher in a workspace \(\mathcal{W} \subset SE(2)\). Without loss of generality, we assume that all objects are $K$-polygons. The robot can push objects through a flat bumper attached at its front. We represent the state of the robot as \(p \in \mathcal{W}\), and the state of each object \(j\in \mathcal{M}=\{1,\dots, m\}\) as \(o_j \in \mathcal{W}\). The robot has rear-axle, simple-car kinematics of the form \(\dot{p} = f(p, u)\), where \(u\) is a control input (speed and steering angle) from a space of controls $\mathcal{U}$. The objective of the robot is to rearrange the objects from their initial poses ${O^s = \left(o_1^s, o_2^s, \dots, o_m^s\right)}$
to their goal poses ${O^g = \left(o_1^g, o_2^g, \dots, o_m^g\right)}$. We seek to develop a planning framework to enable the robot to efficiently rearrange all objects to their goals. We assume the robot has accurate knowledge of: its ego pose; the starting configuration of all objects, $O^s$; the bumper-object friction coefficient. 

\section{ReloPush-BOSS: Optimization-guided Nonmonotone Rearrangement Planning}\label{sec:approach}

We give an overview of our proposed architecture \emph{ReloPush-BOSS}, and then elaborate on its core components.

\subsection{Rearrangement Sequence Planning}\label{sec:backtracking}


An overview of ReloPush-BOSS is illustrated in~\figref{fig:architecture}. At the high level, ReloPush-BOSS involves a sequence planner, implemented as a depth-first search (DFS) over a sequence of rearranged objects. This search starts from depth $d0$, representing the initial configuration of the world, i.e., objects lying at $O^s$, robot lying at $p^s$. A push-traversability (PT)-graph is first constructed, describing how objects in $\mathcal{W}$ can move via stable pushing (see~\figref{fig:ptgraph}). For each object to be rearranged from $o^s_j$ to $o^g_j$, a graph search is performed, resulting in a corresponding rearrangement path. These paths are then ordered wrt the distance required to push-transfer them to their goal pose, and stored in an ordered candidate queue (shorter paths are better). We check if the rearrangement of shortest path is reachable by the robot from its current pose. If reachable, the rearrangement is feasible. Thus, the rearrangement is reflected in an updated world configuration $d1$. If not reachable, the rearrangement is infeasible. In that case, the planner backtracks and tries the next best option in the queue. This whole process is repeated until all objects reach their goals or the candidate queue is exhausted, representing a failure. This cost-ordered depth-first search with backtracking represents an improvement over ReloPush~\citep{ahn2025relopush}, which marks failure when encountering an infeasible rearrangement. While the greedy selection does not guarantee a globally optimal sequence, we employ the strategy to avoid the combinatorial explosion associated with exhaustive sequence planning for global optimum.

\begin{figure}
\vspace{3pt}
    \centering
    \includegraphics[width=\linewidth]{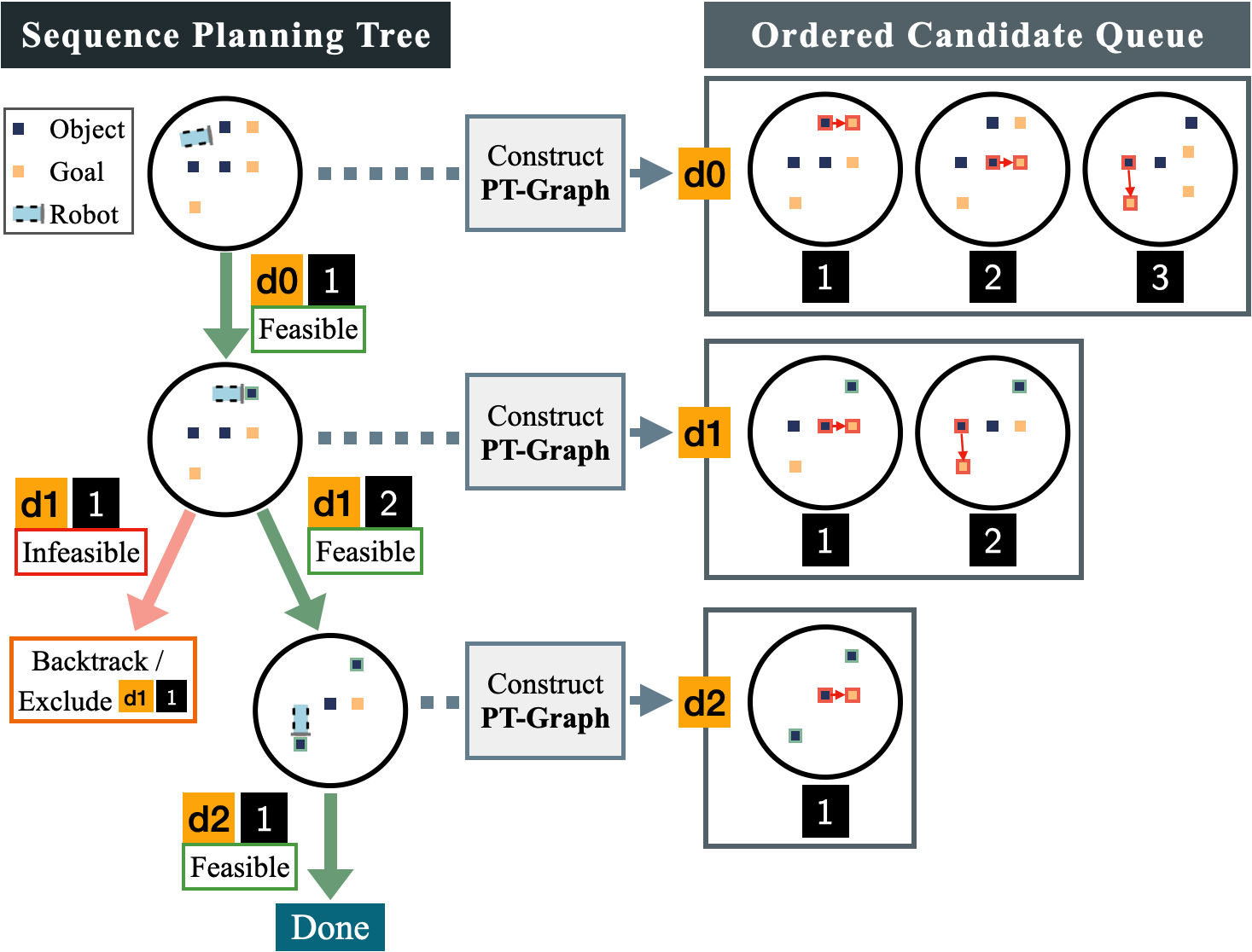}
    \caption{ReloPush-BOSS embeds a PT-graph in a depth-first high-level planner. Starting from the initial configuration (d0), it builds a PT-graph at each depth, enumerates candidate rearrangements, and orders them by cost in a priority queue (ordered candidate queue). Candidates are tested in order; infeasible ones (e.g., inaccessible objects) are pruned and the planner backtracks until it finds a feasible sequence or exhausts all options.}
    \label{fig:architecture}
\end{figure}

\begin{figure}[t]
\begin{subfigure}{0.325\linewidth} 
\includegraphics[width=\linewidth]{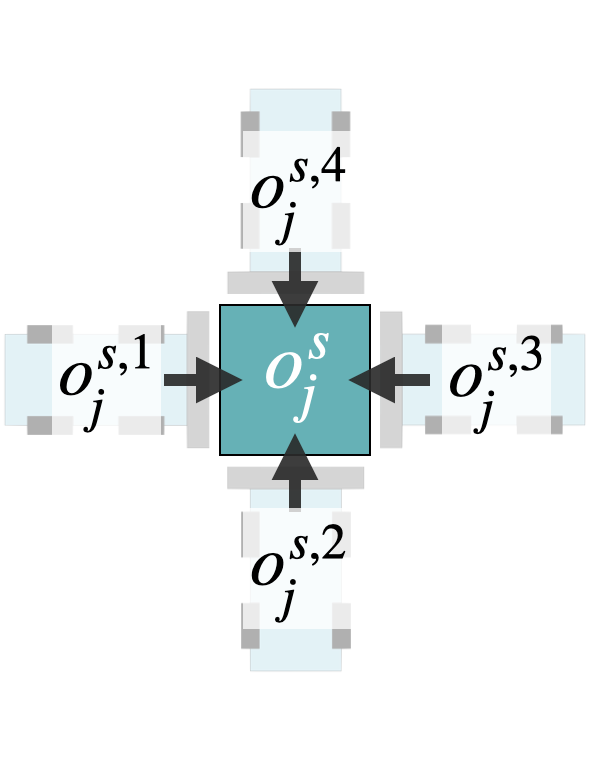}
\caption{Pushing Poses.}
\label{fig:pushing_poses}
\end{subfigure}
\begin{subfigure}{0.325\linewidth} 
\includegraphics[width=\linewidth]{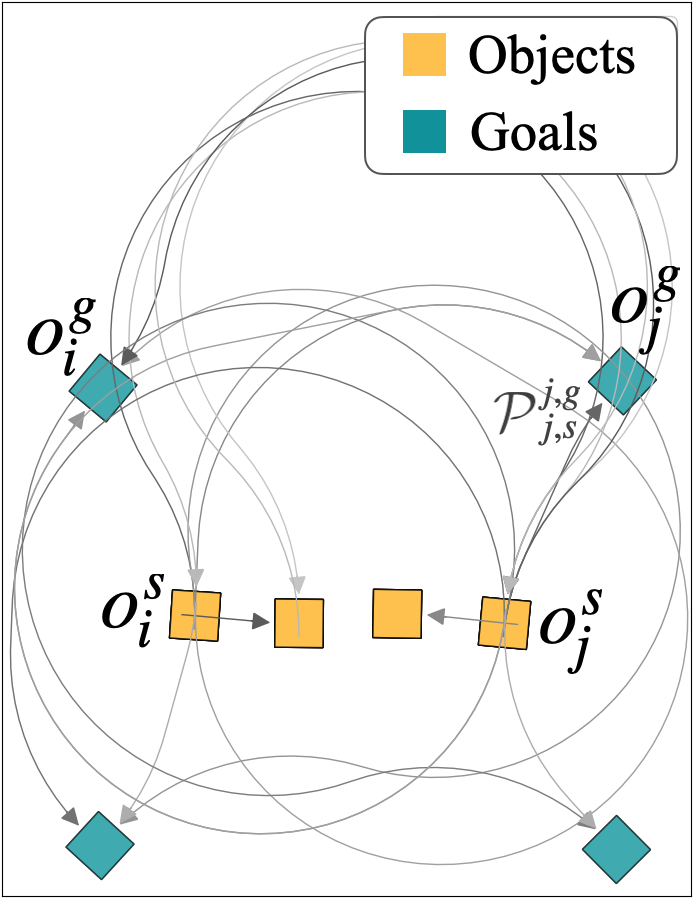}
\caption{Traversability.}
\label{fig:ptgraph_path}
\end{subfigure}
\begin{subfigure}{0.325\linewidth}
\includegraphics[width=\linewidth]{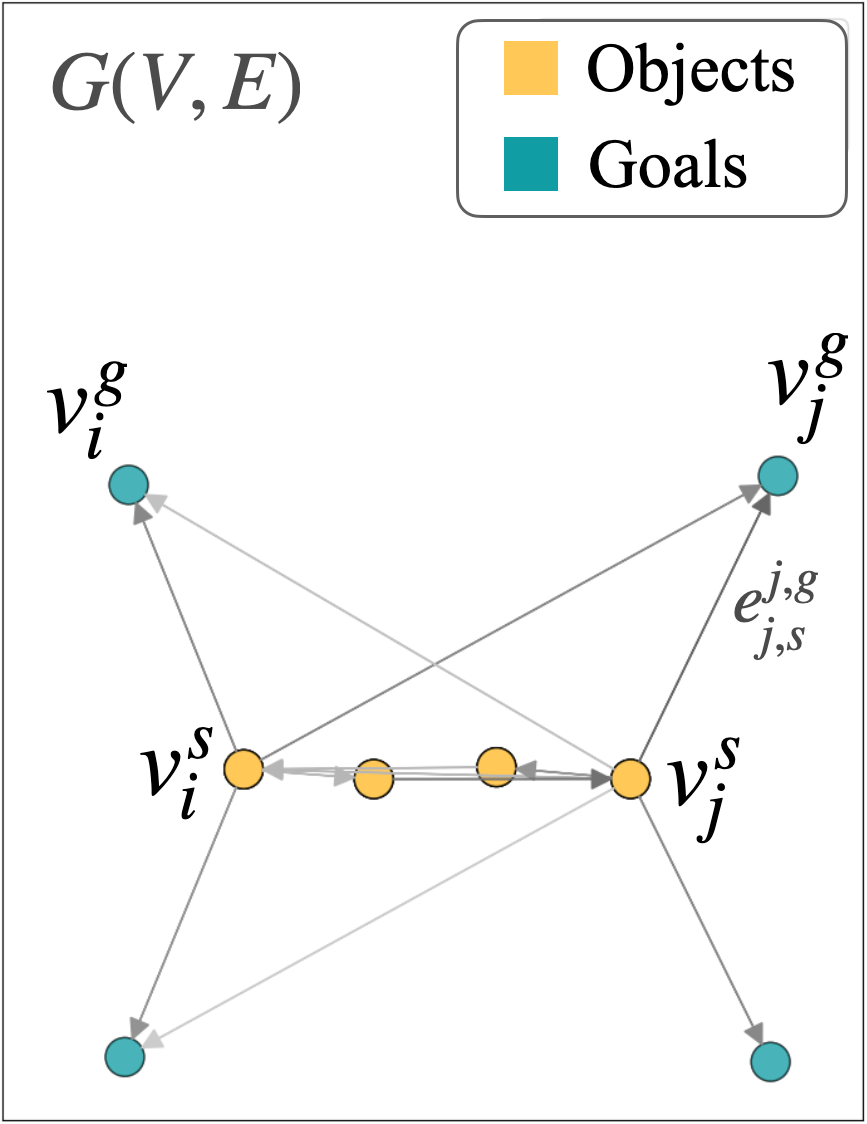}
\caption{Traversability Graph.}
\label{fig:ptgraph_}
\end{subfigure}
    \caption{Push-Traversability graph (PT-graph). Vertices denote pushing poses for objects/goals; directed edges encode feasible pushes weighted by total pushing length. (a) Pushing poses are defined as contact normals on object's sides. (b) Workspace view showing traversable connectivity between objects/goals. (c) Corresponding graph; for clarity, only one vertex per object is shown.}

  \label{fig:ptgraph}
\end{figure}





\subsection{Push-Traversability Graph}

\textbf{Construction}. Object $j$, lying at a pose $o_j$ admits a finite set of \emph{pushing poses} $F(o_j) = \{p_j^1,\dots,p_j^k\}$, defined by stable contact points on the object's boundary and the surface normal in which stable pushing is possible (i.e. K=4 for rectangular object. see~\figref{fig:pushing_poses}). As shown in~\figref{fig:ptgraph}, the PT-graph $G=(V,E)$ comprises a set of vertices $V=V^s\cup V^g$, where $V^s$ represents the set of all pushing poses for all objects when lying at their initial poses $O^s$, and $V^g$ represents the set of all pushing poses for all objects when lying at their goal poses $O^g$. For any pair of vertices $v_{\mathrm{src}}\in V^s$, $v_{\mathrm{dst}}\in V^s\cup V^g$, a \emph{directed} edge $e=(v_{\mathrm{src}},v_{\mathrm{dst}})\in E$ is added iff there exists a collision-free Dubins path~\citep{dubins1957curves} from the pushing pose corresponding to $v_{\mathrm{src}}$ to the pushing pose corresponding to $v_{\mathrm{dst}}$. To ensure stable pushing during push-transfer between two pushing poses, we set the turning radius of the robot $\rho$ to the quasistatic steering limit $\rho_p\leq \rho_{max}$. In a densely cluttered space, a Dubins path under this limit is often not be feasible due to boundary violations or collisions with objects (see~\figref{fig:prereloa}). In such cases, we search for a prerelocation, i.e., an intermediate object pose $o^{pre}_j \in \mathcal{W}$, splitting the rearrangement of object $j$ from $o_j^s$ to $o^g_j$ into a sequence of two collision-free Dubins paths: one from $o_j^{s}$ to $o_j^{pre}$ and another from $o_j^{pre}$ to $o_j^{g}$ (see~\figref{fig:prerelob}). If found, we embed the paths associated with the $o_j^{pre}$ to an edge to connect the two vertices. While in our prior work~\citep{ahn2025relopush}, this search was performed via sampling along the directions of the object's pushing poses, in this paper we describe an optimization-based approach that searches over a continuous region for a prerelocation that minimizes the push-transfer length (see~Sec.~\ref{sec:optimization}, Sec.~\ref{sec:localminima}). For a finite number of vertices, the PT-graph is constructed in polynomial time~\citep{ahn2025relopush}.

\textbf{Search}. 
 For a pair of vertices $v_j^{s,k}$ to $v_j^{g,l}$, describing the rearrangement of object $j$ from $o_j^{s}$ to $o_j^{g}$ via a path from pushing pose $p_j^{s,k}$ to pushing pose $p_j^{g,l}$, a push-transfer path can be found through standard graph search techniques. Note that the initial pushing pose, $l$, can be different than the final one, $k$, since our graph includes prerelocations, i.e., breaks of contact enabling transitions to different pushing poses. If the path contains a vertex other than $v_j^{s,k}$, $v_j^{g,l}$, that vertex represents an object $b\neq j$ at pose $o_b$, blocking the robot's path from $p_j^{s,k}$ to $p_j^{g,l}$. In that case, a straight-line push-transfer path is found to displace object $b$ from $o_{b}$ to make the rearrangement of object $j$ feasible. That path is found by sampling points along the contact normals of $o_{b}$ and finding the closest one that is not in collision with the rearrangement path of object $j$, and with any other objects. This action, referred to as \textit{Obstacle Relocation}, acts as a primitive clearing heuristic; we strictly limit this to straight-line pushes to ensure the inner loop of the search remains efficient, rather than expanding the search space to find optimal poses for secondary objects. Relocation of the blocking object $b$ will be executed before transferring object $j$, which will now pass unobstructed through the region that was previously occupied by object $b$. 

\begin{figure}[h!]
\begin{subfigure}{0.49\linewidth}
\includegraphics[width=\linewidth]{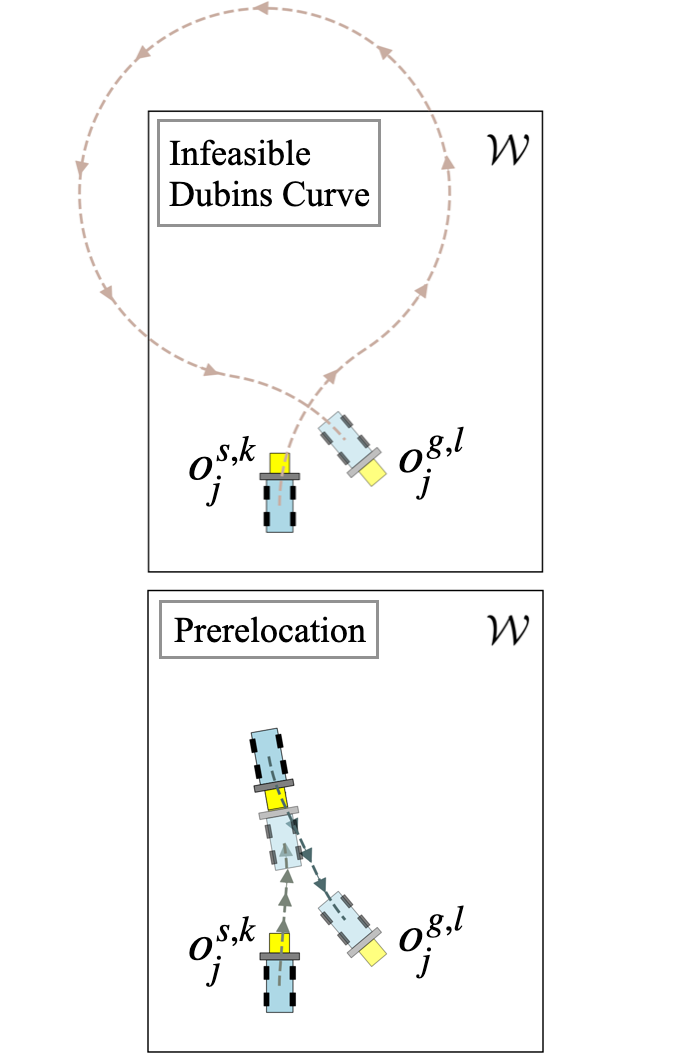}
\caption{}
\label{fig:prereloa}
\end{subfigure}
\begin{subfigure}{0.49\linewidth}
\includegraphics[width=\linewidth]{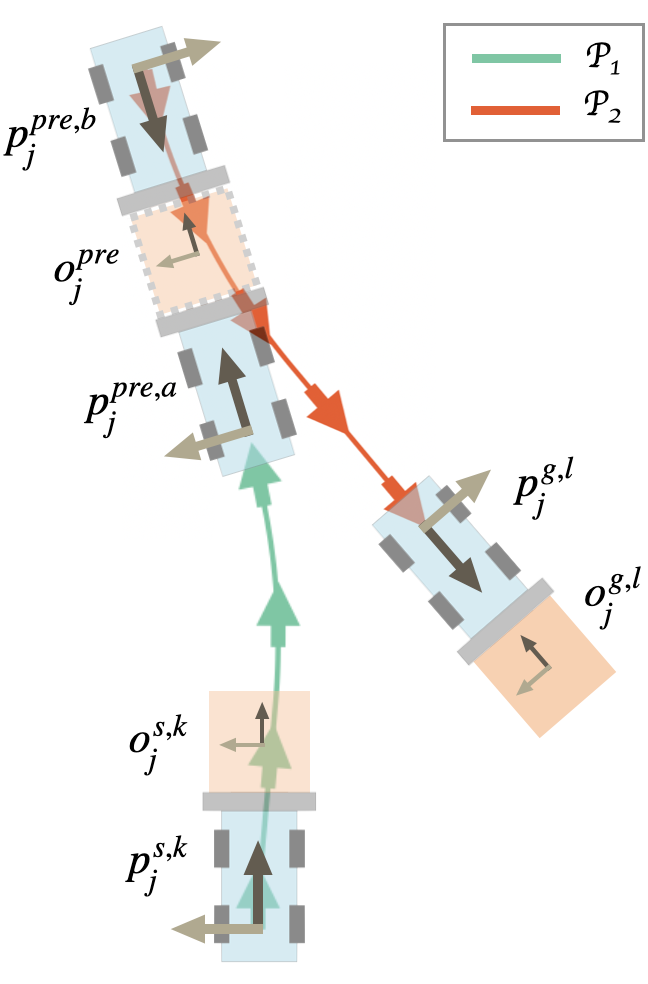}
\caption{}
\label{fig:prerelob}
\end{subfigure}
    \caption{Illustration of a prerelocation. (\subref{fig:prereloa}) Top: A continuous push-transfer Dubins path from $o^{s,k}_j$ to $o^{g,l}_j$ is infeasible; bottom: transitioning to a new pushing pose when the object is transferred to a prerelocation enables a feasible push transfer. (\subref{fig:prerelob}) Zoom in on a prerelocation. A robot path $\mathcal{P}_1$, from $p_j^{s,k}$ to $p_j^{pre,a}$ (shown in green) transfers the object from $o_j^{s,k}$ to $o_j^{pre}$. After the robot transits to a new pushing pose $p_j^{pre,b}$ (transit path not shown), $\mathcal{P}_2$ transfers the object to $o_j^{g,l}$.}
  \label{fig:prerelo}
\end{figure}

\subsection{Optimizing Prerelocations}\label{sec:optimization}

ReloPush-BOSS optimizes a prerelocation wrt the total length of the push-transfer portion of the rearrangement given a start pose, a goal pose, and a minimum turning radius; transit paths to and from objects are not accounted for. For object $j$ moved from $o_j^{s}$ to $o_j^{g}$ through a transition from pushing pose $p_j^{s,k}$ to pushing pose $p_j^{g,l}$, an optimized prerelocation $o_j^{\mathrm{pre}}$ is obtained by solving the following optimization (see~\figref{fig:prerelob}):

\begin{align}C(o_j^{pre}) = \arg &\min \mathcal{L}(\mathcal{P}_1) + \mathcal{L}(\mathcal{P}_2)\label{eq:argmin}\\ 
s.t. \: \mathcal{P}_1 &= Dubins(p^{s,k}_j,p^{pre,a}_j)\label{eq:dubins1}\\
\mathcal{P}_2 &= Dubins(p^{pre,b}_j,p^{g,l}_j)\label{eq:dubins2}\\
\mathcal{P}_1 &, \mathcal{P}_2 \in\mathcal{W}\label{eq:workspace}\\
\mathcal{P}_1 &, \mathcal{P}_2 \cap O =\emptyset \mbox{,}\label{eq:objects}
\end{align}
where $C(\cdot)$, $\mathcal{L}(\cdot)$, $Dubins(\cdot,\cdot)$ are the cost of a \emph{prerelocation}, the length of a path, and a function describing a Dubins path between the poses passed as its arguments, respectively. $p_j^{pre,a}$ is the pose of the robot when arriving at $o_j^{pre}$, $p_j^{pre,b}$ is the pose of the robot when initiating transfer to $o_j^{g}$. Eqs \eqref{eq:workspace} and \eqref{eq:objects} enforce respectively the constraints that these segments are within the workspace boundary and not in collision with any objects. In this work, we focus on analyzing the finding of at most one prerelocation per object as optimization for sequential prerelocations would transform the problem into a different optimization problem. Our analysis confirms that a single prerelocation shows considerable effectiveness in finding feasible solutions that pose lower costs than those of the prior work.

\subsection{Avoiding High-Cost Local Minima}\label{sec:localminima}



We show that a gradient-based optimization over the cost of eq.~\eqref{eq:argmin} is prone to local minima, and propose a technique for warm-starting an optimizer from a high-quality solution. 



\textbf{Dubins path classification}. A Dubins path from a pose $p^s$ to a pose $p^g$ comprises exactly three segments and is either of type \textit{CSC} or \textit{CCC}~\citep{shkel2001classification}, where \textit{S} represents a straight line segment, and $C\in\{L, R\}$ represents an arc segment under turning radius $\rho$, where \textit{L} is a left turn and \textit{R} is a right turn. 
As shown in~\figref{fig:dubins}, when the required heading change between $p^s$ and $p^g$ is too large compared to their respective position distance, $d$, only \textit{CCC} paths are admissible; these tend to be long as they feature three consecutive turns~\citep{shkel2001classification}. 


\begin{figure}[t]
\centering
\begin{subfigure}{0.43\linewidth}
\includegraphics[width=\linewidth]{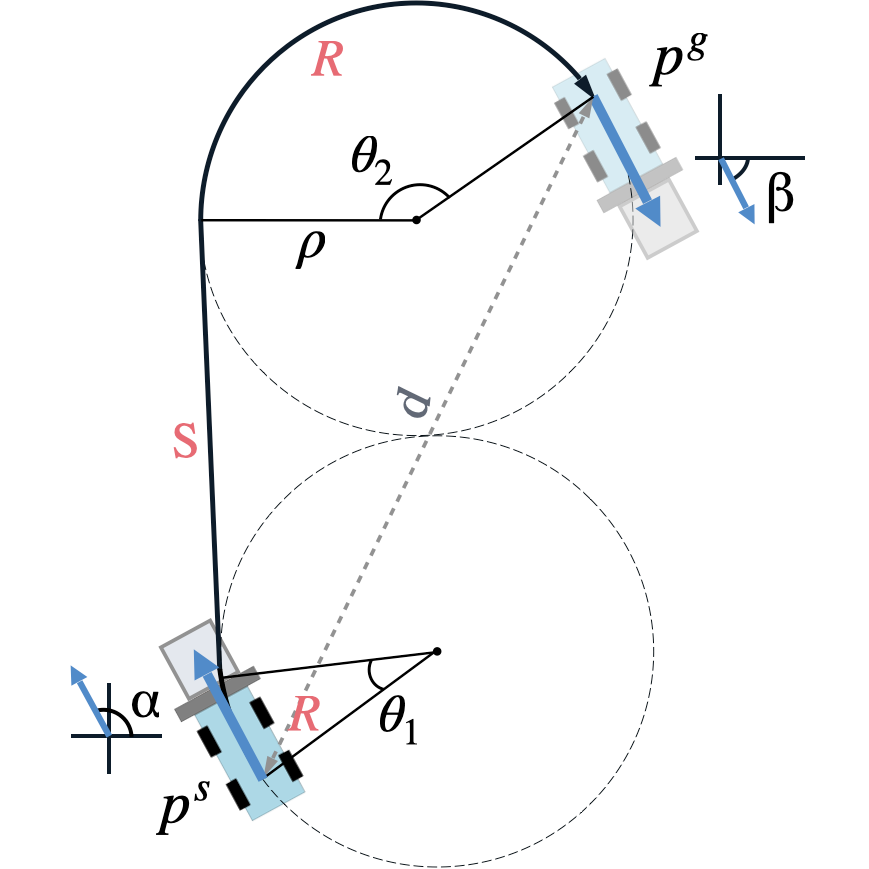}
\caption{CSC Dubins path.}
\label{fig:dubins_a}
\end{subfigure}
\begin{subfigure}{0.43\linewidth}
\includegraphics[width=\linewidth]{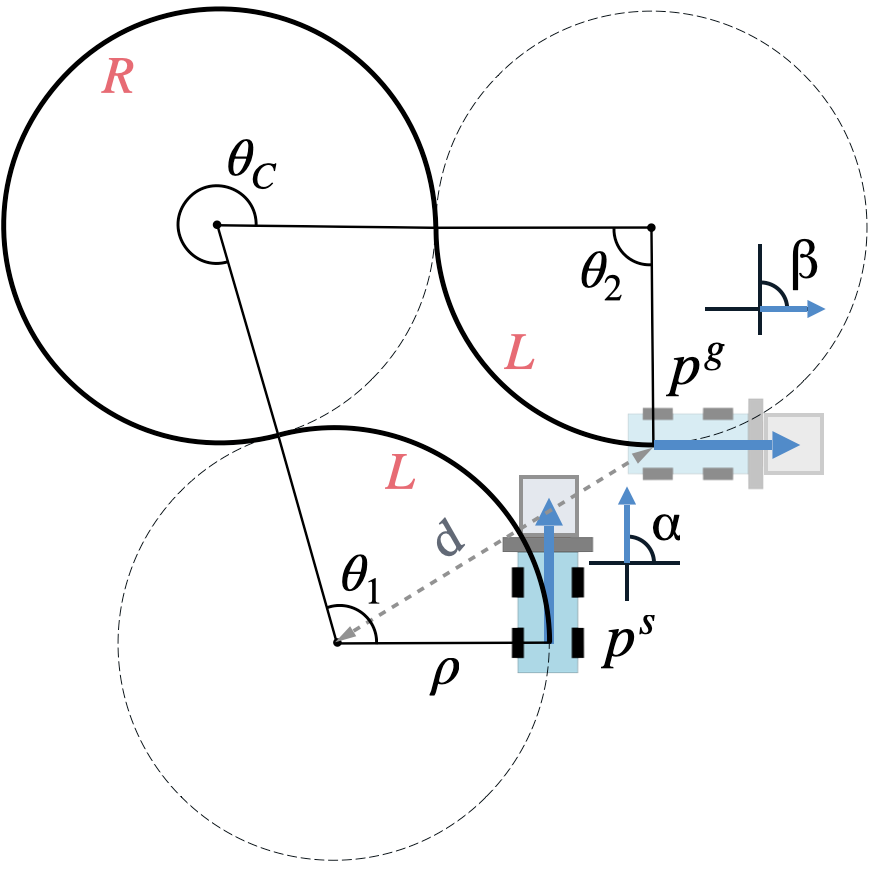}
\caption{CCC Dubins path.}
\label{fig:dubins_b}
\end{subfigure}
    \caption{For a car-like robot, a Dubins path describes the optimal path of transitioning from a start pose $p^s$ to a goal pose $p^g$. A Dubins path can be of two different type families, i.e., $CSC$ or $CCC$, where $S$ is a a straight line segment, and $C\in \{L, R\}$ denotes a left or right arc. When the distance $d$ is too small relative to the respective heading difference, only $CCC$ paths are admissible (the short-distance regime), yielding longer paths with a wider spatial sweep \citep{shkel2001classification}. Circles indicate minimum–turning-radius arcs; $\alpha$ and $\beta$ are the start and goal headings.
}
  \label{fig:dubins}
\end{figure}

\textbf{Non Convexity of $n$-point Dubins path}. \citet{goaoc2013bounded} study the problem of computing the shortest path of bounded curvature visiting a predefined sequence of $n$ points in the plane. They show that the length of this path as a function of the local curve orientation at these points is locally convex, but globally non-convex, indicating the existence of multiple local minima. 
The prerelocation optimization problem of eq.~\eqref{eq:argmin} represents a more general version, where the prerelocation is optimized with respect to the resulting total curve length, $\mathcal{L}(\mathcal{P}_1)+\mathcal{L}(\mathcal{P}_2)$, over both its \emph{position} and orientation. Thus, it is also non-convex.

\begin{figure}[t]

  \begin{subfigure}{0.495\linewidth}
    \includegraphics[width=\linewidth]{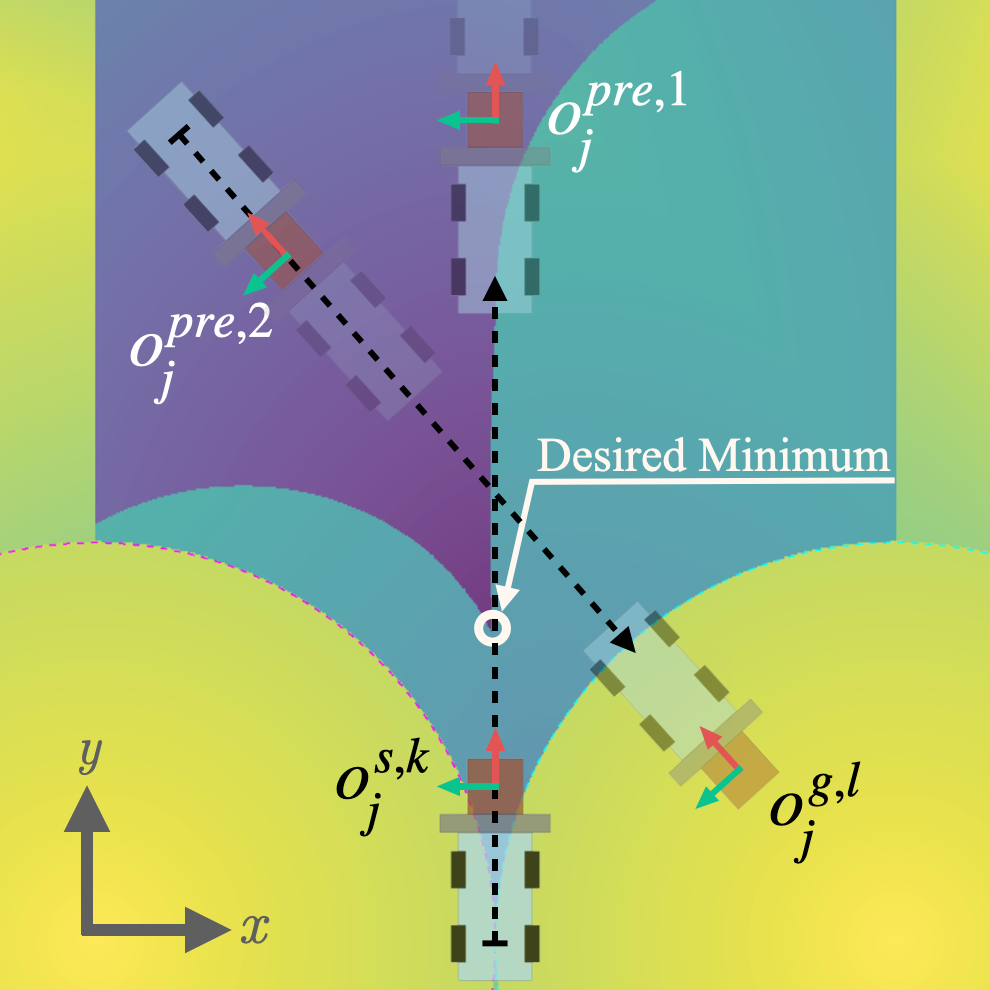}
    \caption{Prerelocation Optimization.}
    \label{fig:combined_cost1}
  \end{subfigure}
  \hfill
  \begin{subfigure}{0.495\linewidth}
    \includegraphics[width=\linewidth]{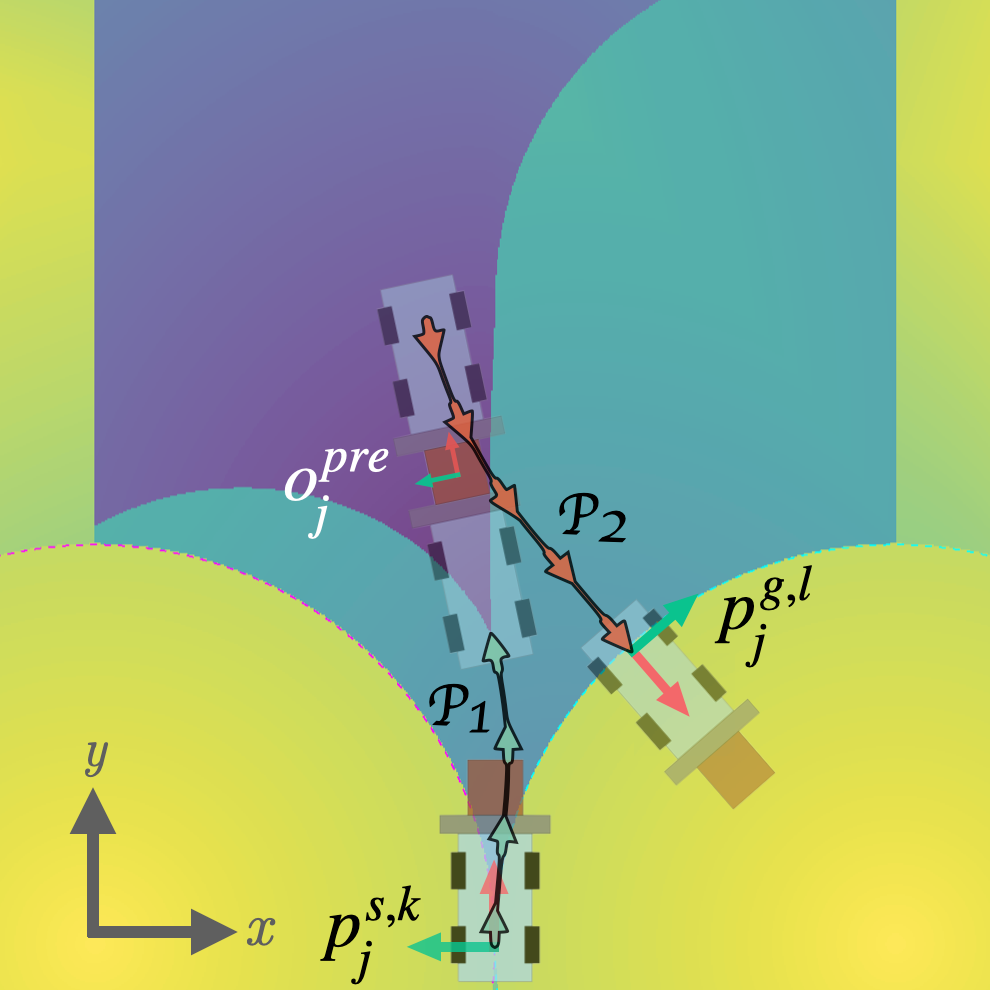}
    \caption{Optimized solution.}
    \label{fig:combined_cost2}
  \end{subfigure}

  \caption{Cost landscape of $\mathcal{L}(\mathcal{P}_1)+\mathcal{L}(\mathcal{P}_2)$ plotted over $\mathbb{R}^2$. Darker colors represent lower costs. For visualization purposes, the plot assumes a fixed heading at each position but the proposed optimization minimizes over $SE(2)$. (\subref{fig:combined_cost1}) $o^{pre,1}_j$ and $o^{pre,2}_j$ are two possible initializations following the approximation discussed in~Sec.~\ref{sec:localminima}. The former is constructed through a sequence of $CS$ and $S$ curves, whereas the latter through a sequence of $S$ and $CS$ curves. Descending from either seed converges to a minimum of $\mathcal{L}(\mathcal{P}_1)+\mathcal{L}(\mathcal{P}_2)$. (\subref{fig:combined_cost2}). Visualization of the optimized solution. 
  }
  \label{fig:three-subfig}
\end{figure}

\begin{figure}[t]

  \begin{subfigure}{0.495\linewidth}
    \includegraphics[width=\linewidth]{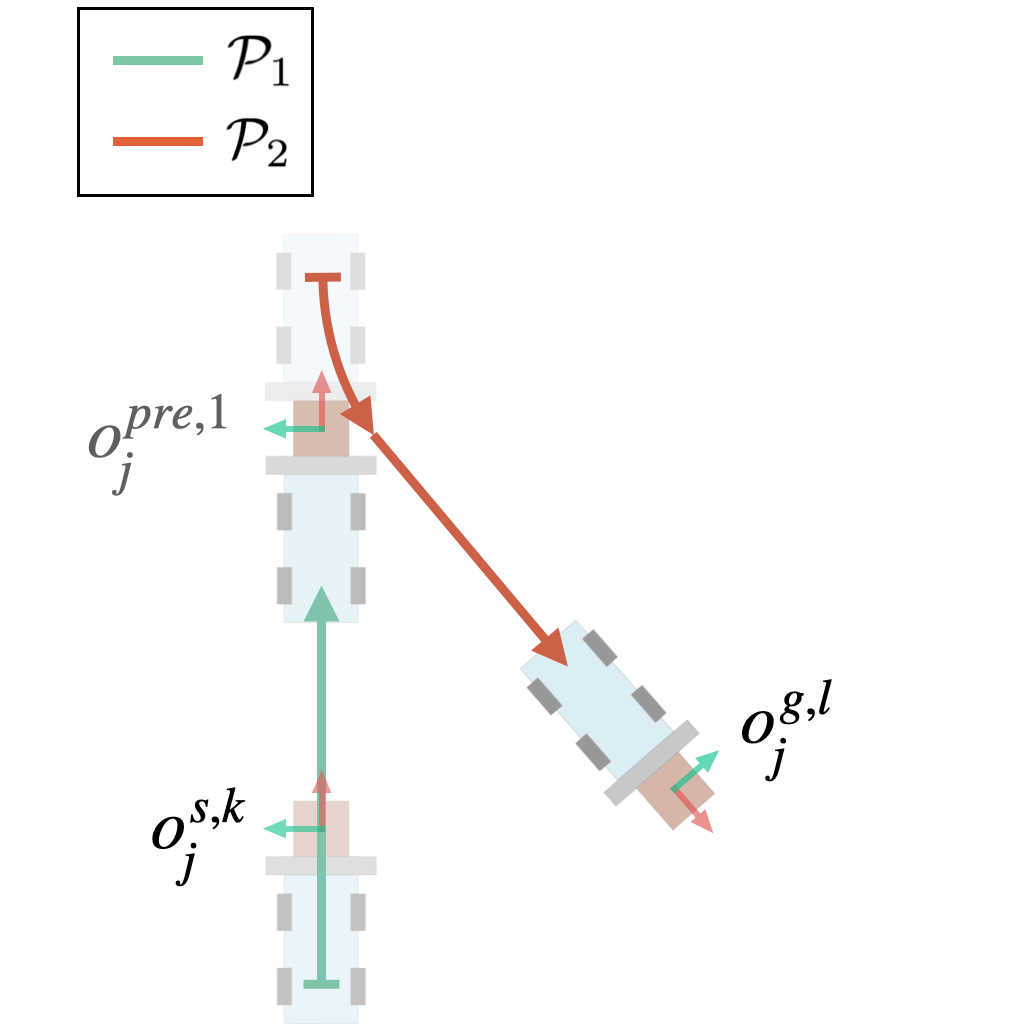}
    \caption{Initial guess.}
    \label{fig:rotate_fig1}
  \end{subfigure}
  \hfill
  \begin{subfigure}{0.495\linewidth}
    \includegraphics[width=\linewidth]{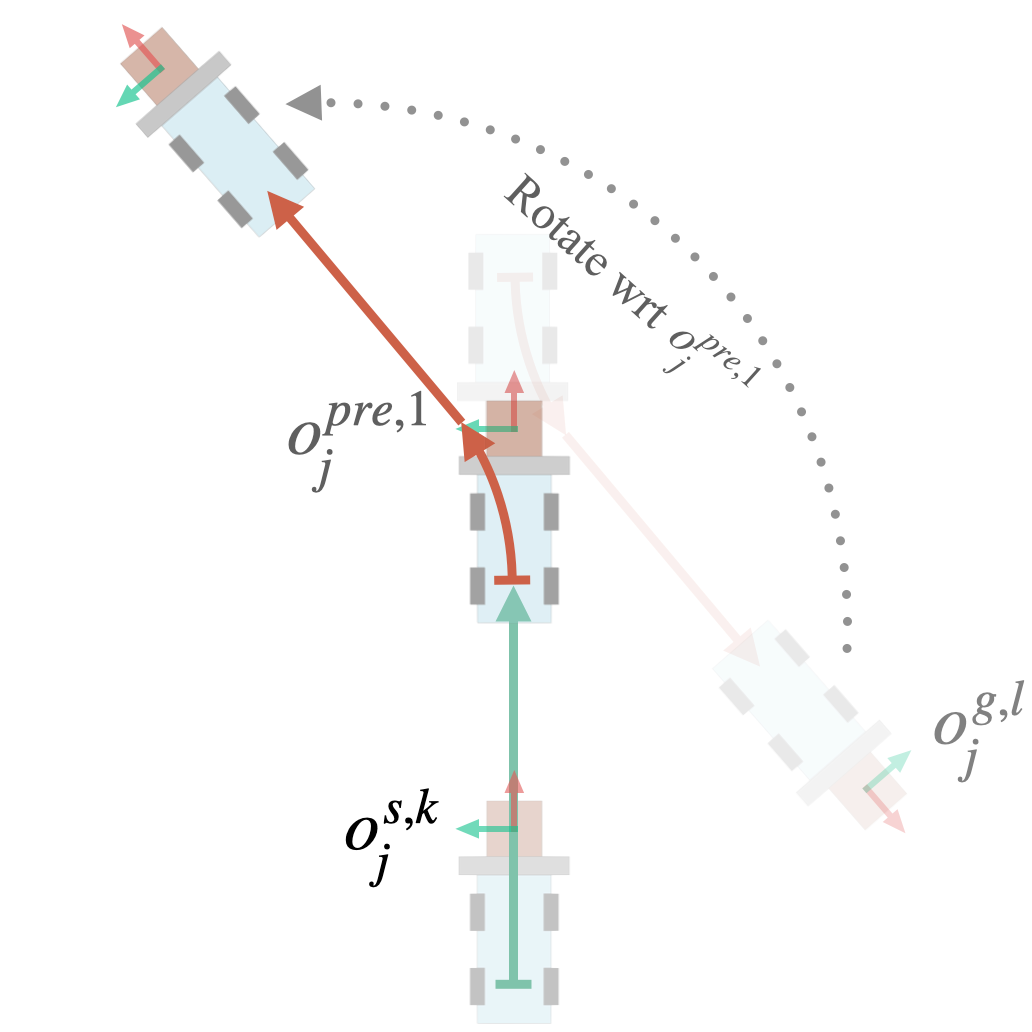}
    \caption{Conversion to 3PDP~\citep{CHEN2019368}.}
    \label{fig:rotate_fig2}
  \end{subfigure}

  \caption{Initial guess as a sequence of $\mathcal{P}_1$ ($S$) and $\mathcal{P}_2$ ($CS$) curves. By rotating $\mathcal{P}_2$ and $p^{g,l}$  wrt $o_j^{pre}$ such that the orientation of $p_j^{pre,a}$ matches the orientation of the rotated $p_j^{pre,b}$, we convert our problem to the 3PDP~\citep{sadeghi2016efficient,CHEN2019368}, enabling us to leverage a high-quality approximation of the optimum.}
  \label{fig:rotate_fig}
\end{figure}
\textbf{Warm-starting the optimization}. For the $n$-point Dubins path problem, considering fixed positions at a $d>4\rho$ distance apart and varying orientations at the intermediate points, \citet{goaoc2013bounded} show that an approximation of the shortest path (at most 1.91 times of the optimum) can be found by connecting consecutive waypoints with $S$ and $CS$ curves. Since prerelocation optimizes a full pose rather than an orientation, it induces a higher freedom for path optimization, offering a potential to reduce the 1.91 factor. Nevertheless, we cannot directly leverage this approximation because unlike the original $n$-point Dubins path problem, a prerelocation, by definition, results in a break of contact corresponding to a change in the pushing pose. 

We convert our \emph{prerelocation} optimization to a 3-point Dubins path problem (3PDP~\citep{CHEN2019368}) to leverage this result in the construction of an optimization seed. First, we find the object pose $o_j^{pre,1}$ enabling a straight-line path $\mathcal{P}_2$ ($S$) from $p_j^{pre,b}$ to $p_j^{g,l}$ and a ($SC$) curve $\mathcal{P}_1$ from $p_j^{s,k}$ to $p_j^{pre,a}$. By symmetry, we could also define a similar initial guess comprising a $S$ segment for $\mathcal{P}_1$, and a $SC$ segment for $\mathcal{P}_1$. We show both solutions in~\figref{fig:combined_cost1}. Next, we rotate $\mathcal{P}_2$ and $p_j^{g,l}$ wrt $o_j^{pre}$ such that the orientation of $p_j^{pre,a}$ matches the orientation of the rotated $p_j^{pre,b}$ (\figref{fig:rotate_fig}). We show that at this configuration represents a high-quality upper bound of an optimal solution.

\begin{theorem}  \label{lem:atmost}
    Let the straight segments of $\mathcal{P}_1$ and $\mathcal{P}_2$ of $o_j^{pre,1}$ be $S_1$ and $S_2$, respectively. Then $C(o_j^{pre,1})$ is larger than an optimal cost by at most $\mathcal{L}(S_1)+\mathcal{L}(S_2)$.
\end{theorem}
\begin{proof}
    To meet the orientation goal of rearranging $o_j^{s,k}$ to $o_j^{g,l}$, the robot is required to rotate the object by $\Delta\theta=|\theta_j^{s,k}-\theta_j^{g,l}|$ where $\theta_j^{s,k}$ and $\theta_j^{g,l}$ are the orientation of $o_j^{s,k}$ and $o_j^{g,l}$, respectively. Because the robot is constrained by the minimum turning radius $\rho$, the robot needs to turn along its turning arc of at least $\rho\cdot\Delta\theta$ length. Thus, any valid solution of $o_j^{pre}$ must incur a travel distance of at least $\rho\cdot\Delta\theta$, which is the same length as the only $C$ segment of $o_j^{pre,1}$. Therefore, the total length $o_j^{pre,1}=\mathcal{L}(\mathcal{P}_1) + \mathcal{L}(\mathcal{P}_2)$ is at most $\mathcal{L}(S_1)+\mathcal{L}(S_2)$ larger than the optimal cost.
\end{proof}

 In our confined settings, the straight segments are generally short, and concurrently, each Dubins path generally falls into $d<4\rho$. As discussed in prior work, the length of a Dubins path is discontinuous when $d<4\rho$ at the boundary between $CSC$ and $CCC$ \citep{vana2018optimal,shkel2001classification}, with a $CCC$ solution being strictly longer than any $CSC$. As the initialization $o_j^{pre,1}$ is a solution with a straight path in both $\mathcal{P}_1$ and $\mathcal{P}_2$, it is in the $CSC$ region for both paths. While the region is generally not convex, the Dubins path is analytic with respect to $d$ and start/goal orientations as long as the path type remains consistent~\citep{shkel2001classification}. Thus, a gradient-based optimizer will converge to a first-order stationary point, a local minimum that generally poses shorter total length than that of $o_j^{pre,1}$. By the symmetry of the construction, \emph{Theorem}~\ref{lem:atmost} also holds for the other initialization, $o_j^{pre,2}$.

 \textbf{Conditional Completeness}. The overall method is \emph{conditionally complete}, conditioned on the existence of a solution within the search space defined by these primitives: at most one prerelocation, straight-line obstacle clearing, and standard Dubins connections. Since these constraints ensure a finite number of candidate sequences, our DFS complete within the search space.

\begin{figure*}[ht]
    \centering
    \makebox[\textwidth][c]{%
        \begin{subfigure}[b]{0.16\textwidth}
            \centering
            \includegraphics[width=\textwidth]{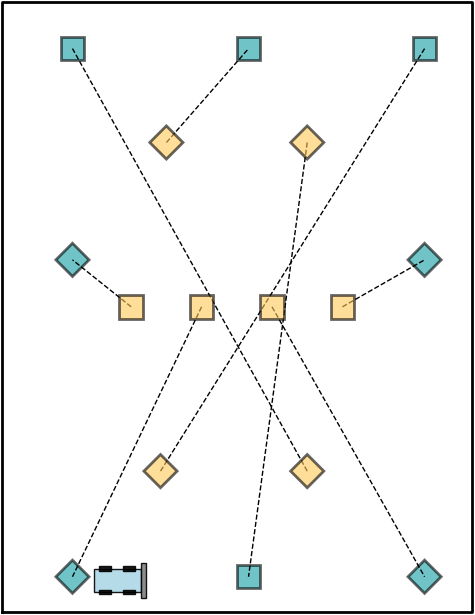}
            \includegraphics[width=\textwidth]{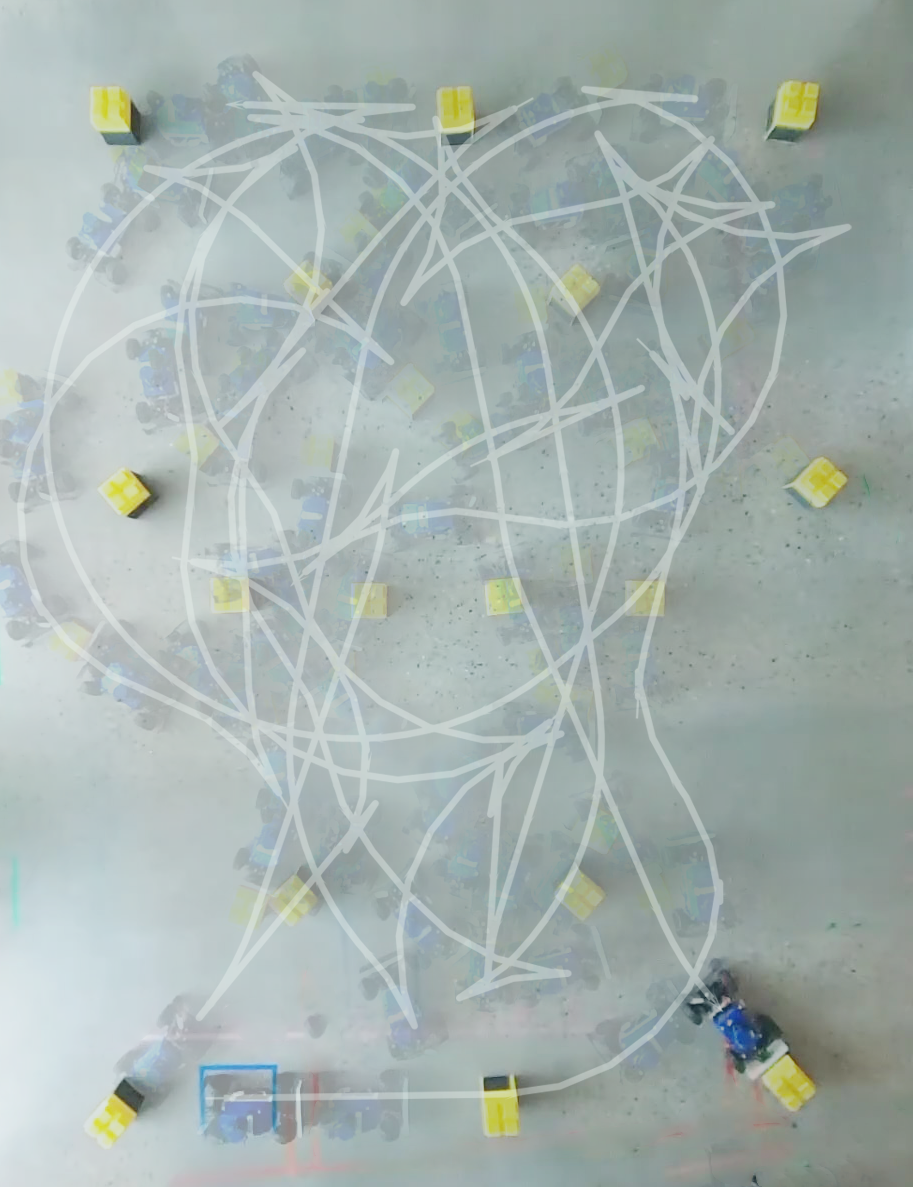}
            \caption{$m=8$}
        \end{subfigure}
        \hfill
        \begin{subfigure}[b]{0.16\textwidth}
            \centering
            \includegraphics[width=\textwidth]{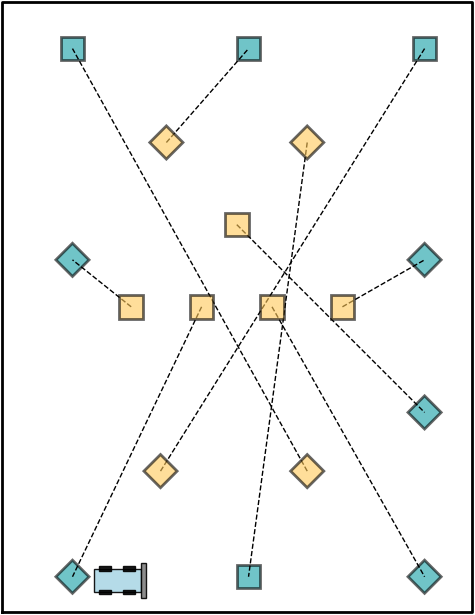}
            \includegraphics[width=\textwidth]{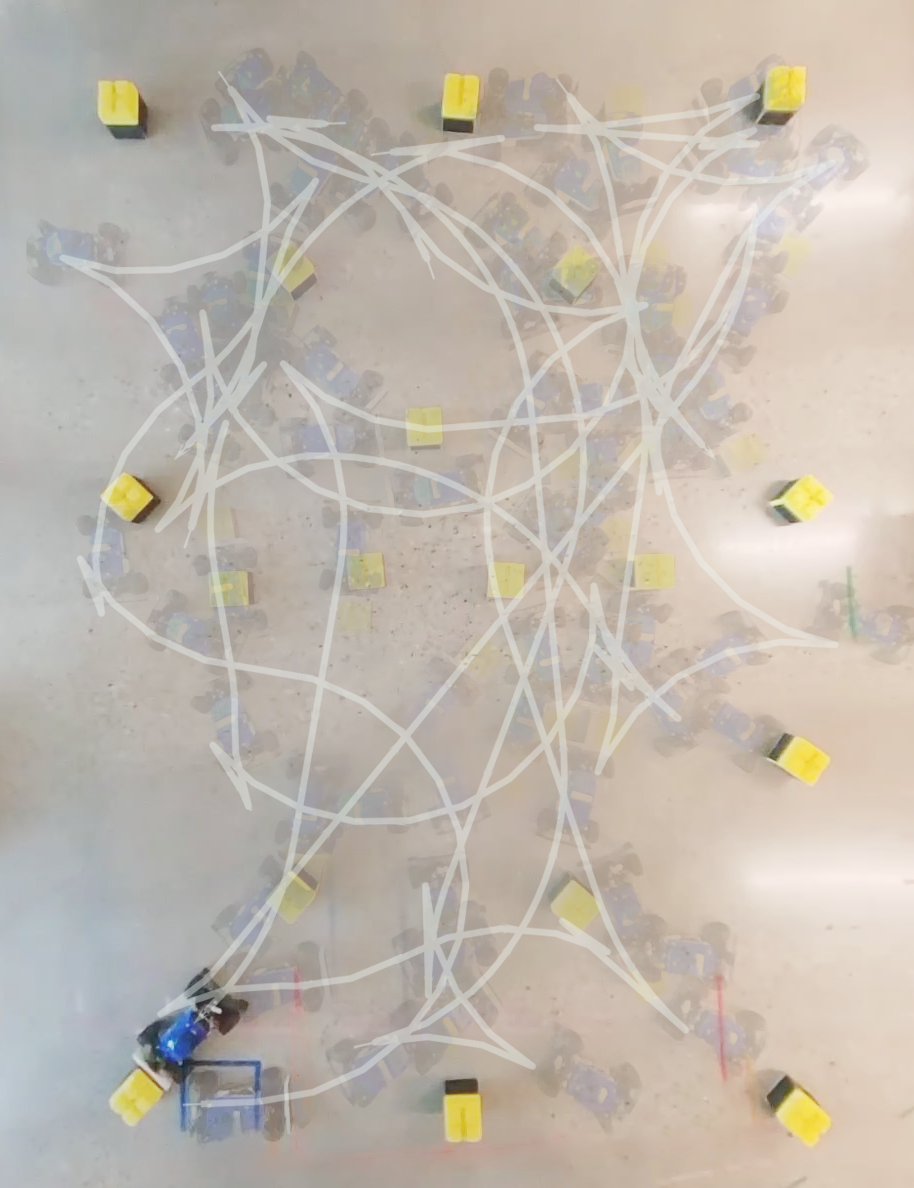}
            \caption{$m=9$}
        \end{subfigure}
        \hfill
        \begin{subfigure}[b]{0.16\textwidth}
            \centering
            \includegraphics[width=\textwidth]{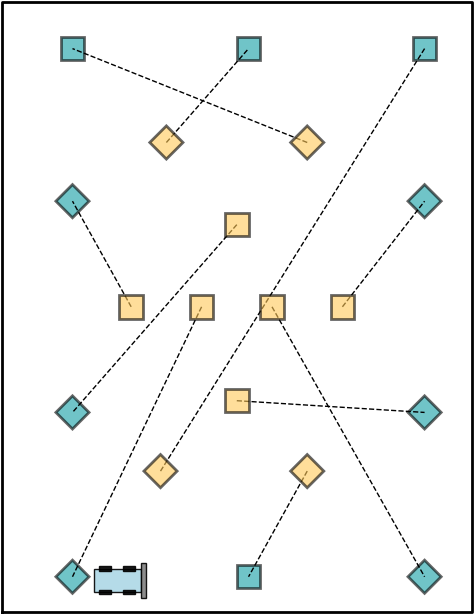}
            \includegraphics[width=\textwidth]{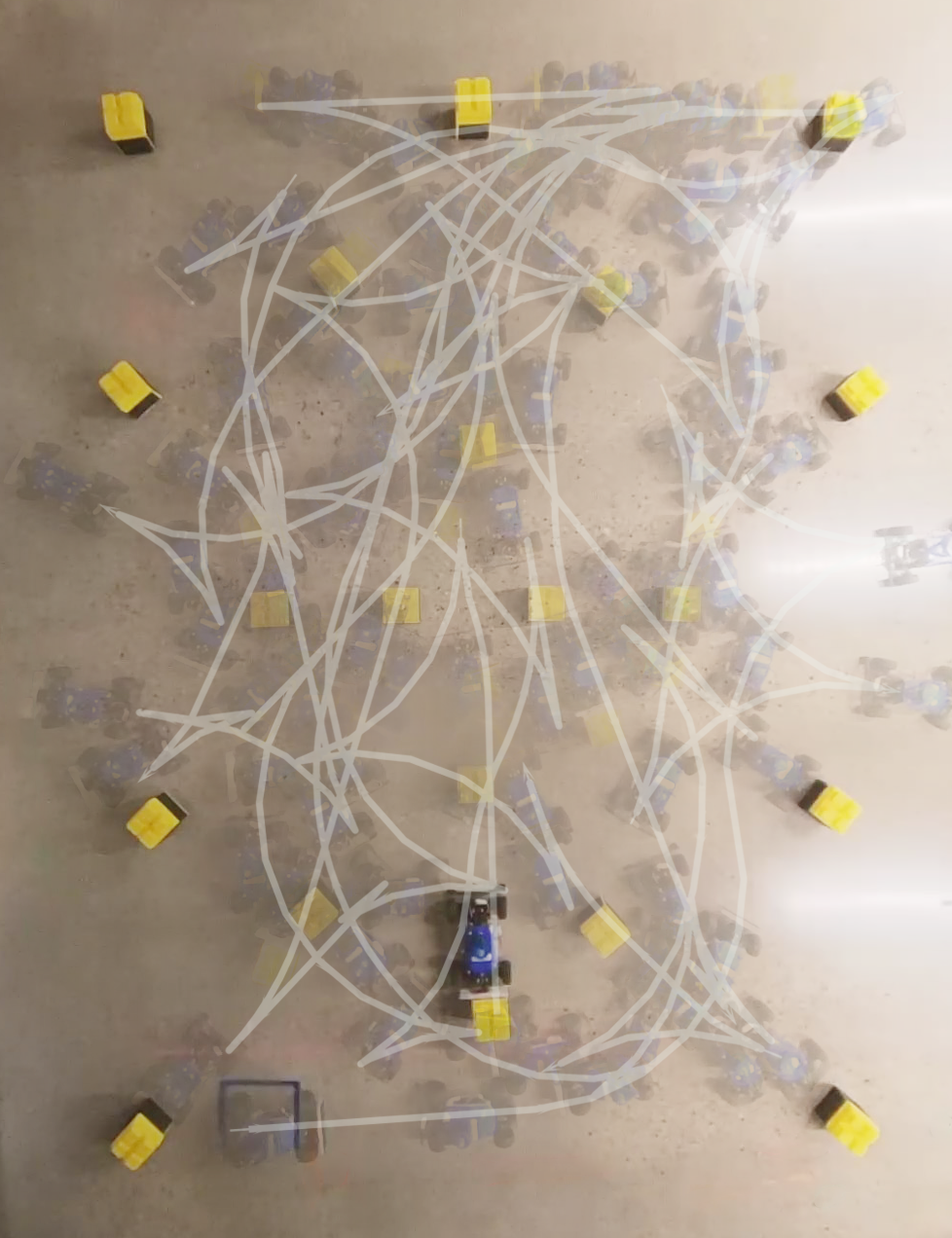}
            \caption{$m=10$}
        \end{subfigure}
        \hfill
        \begin{subfigure}[b]{0.16\textwidth}
            \centering
            \includegraphics[width=\textwidth]{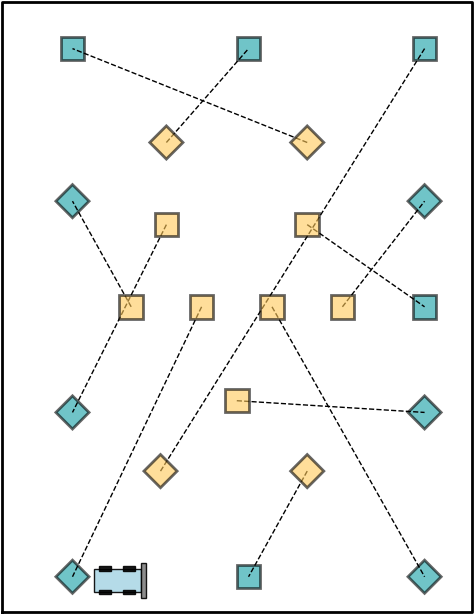}
            \includegraphics[width=\textwidth]{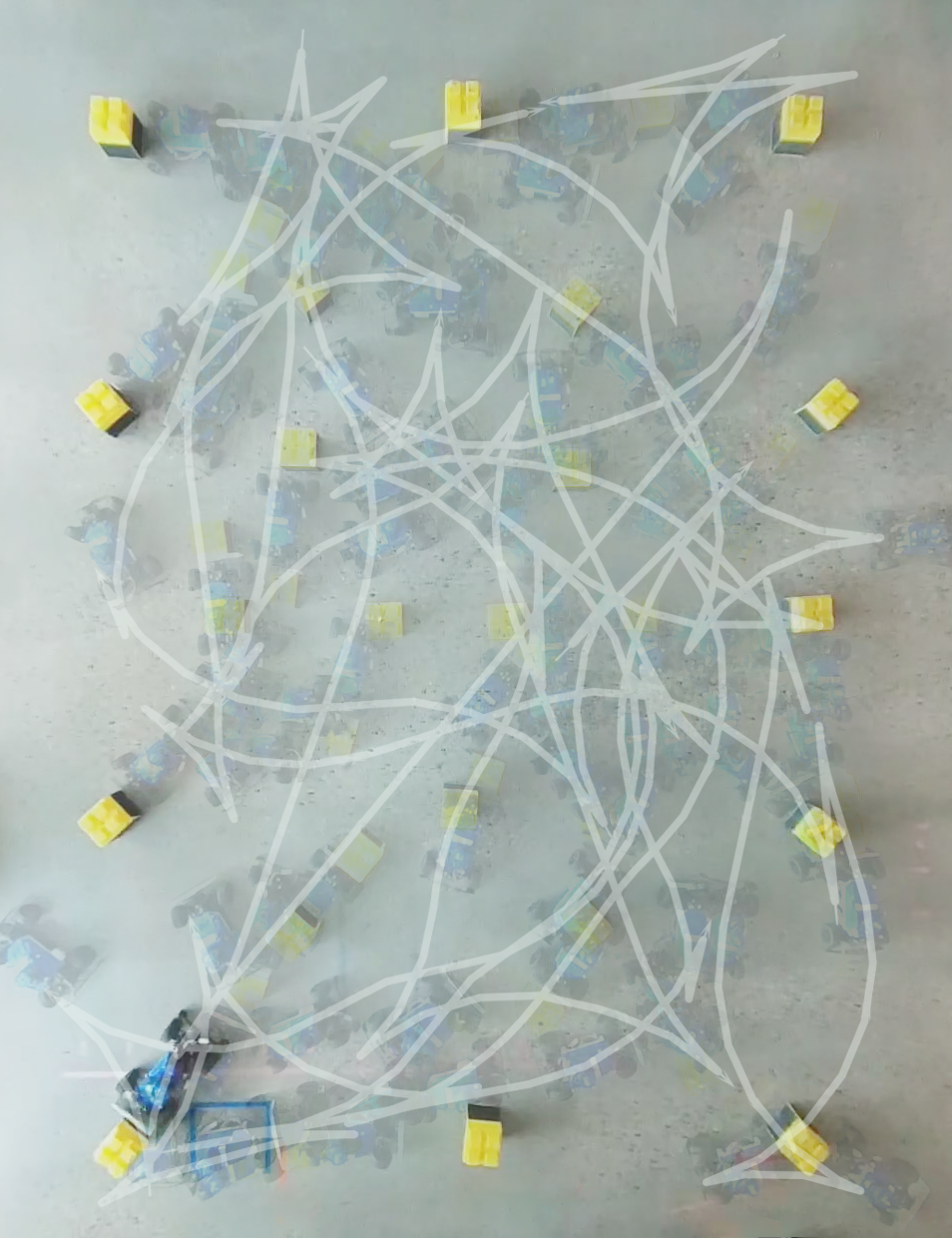}
            \caption{$m=11$}
        \end{subfigure}
        \hfill
        \begin{subfigure}[b]{0.16\textwidth}
            \centering
            \includegraphics[width=\textwidth]{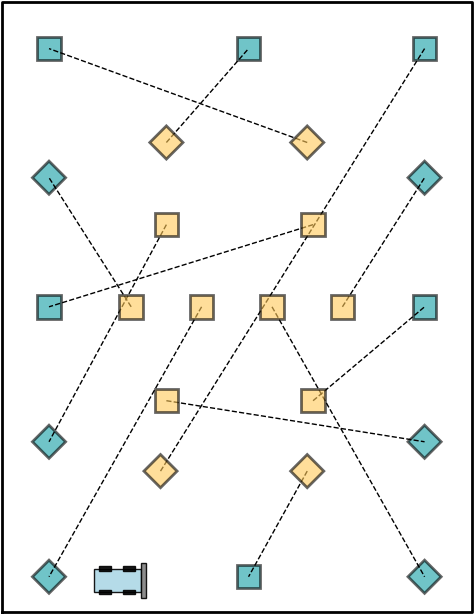}
            \includegraphics[width=\textwidth]{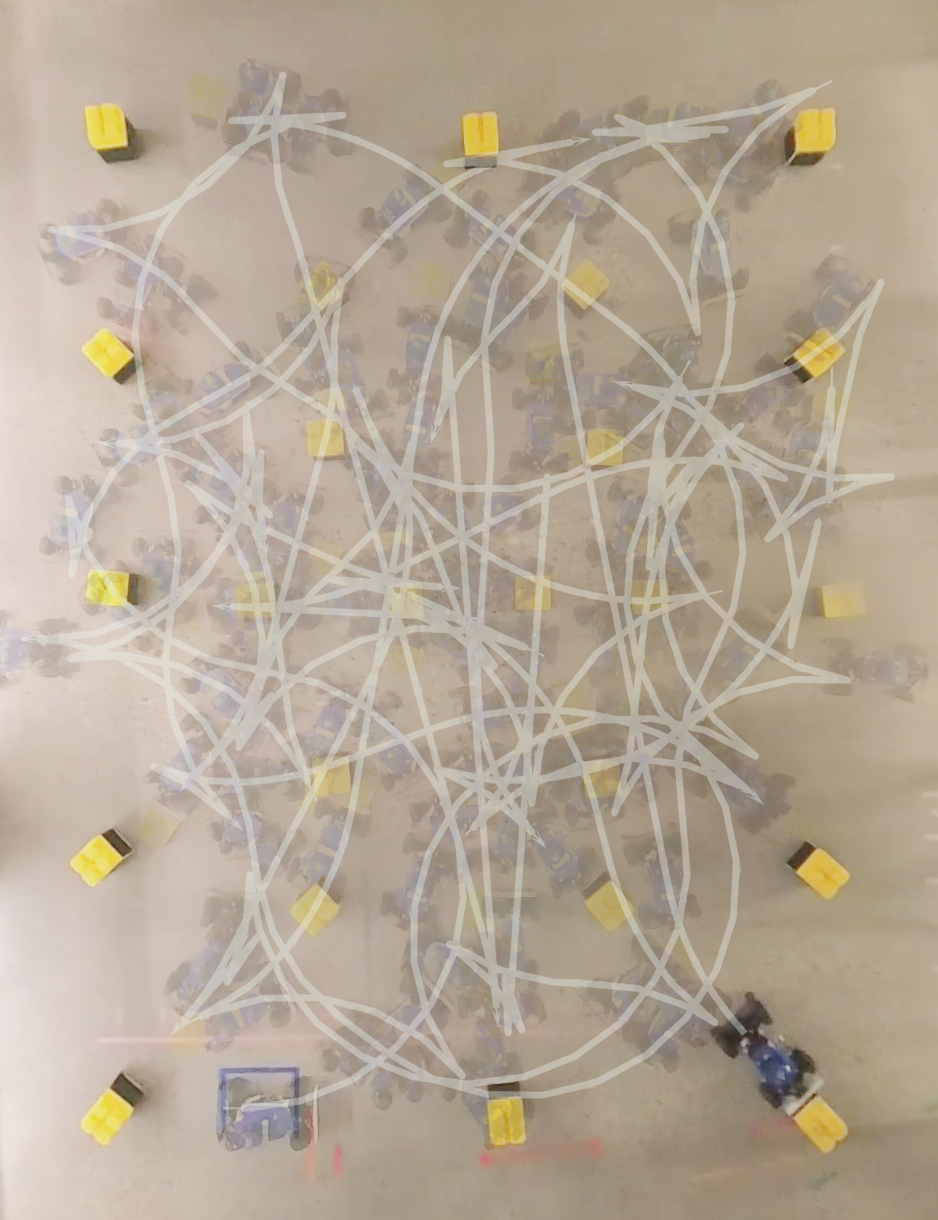}
            \caption{$m=12$}
        \end{subfigure}
        \hfill
        \begin{subfigure}[b]{0.16\textwidth}
            \centering
            \includegraphics[width=\textwidth]{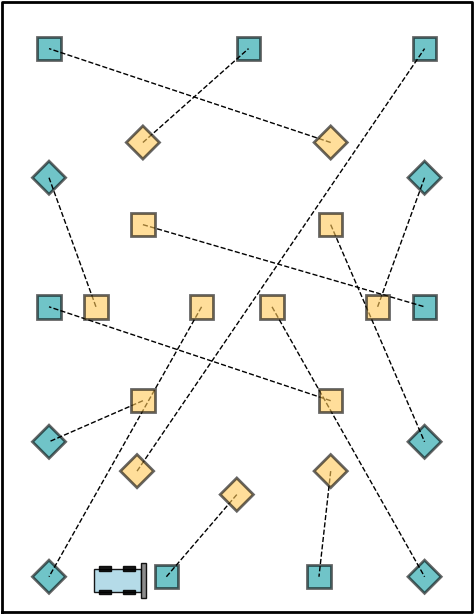}
            \includegraphics[width=\textwidth]{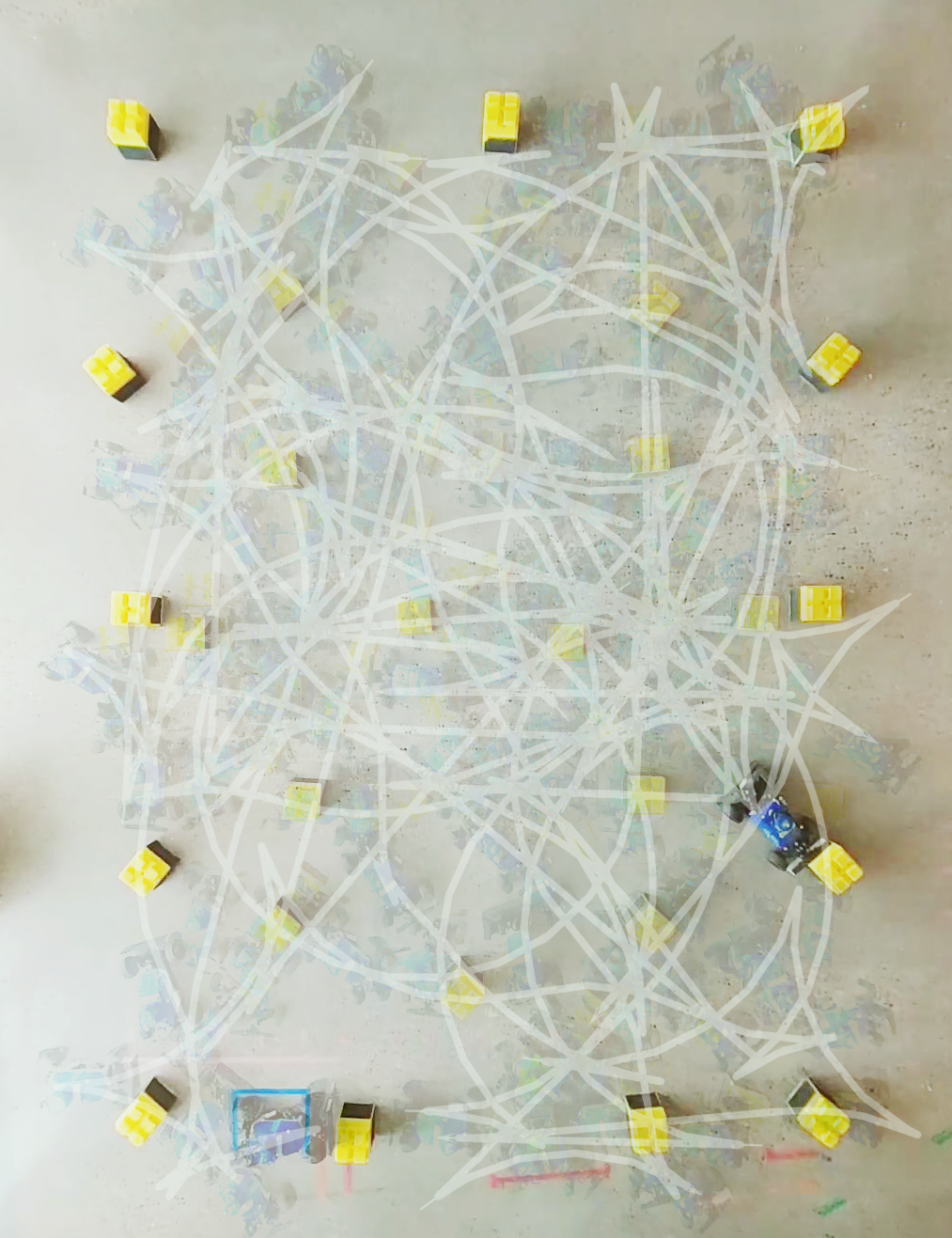}
            \caption{$m=13$}
        \end{subfigure}
    }
    \caption{Top: evaluation scenarios involving $8\sim13$ objects; for each scenario, dotted lines connect initial (blue) and goal (yellow) object poses. Bottom: frame stacks from real-world execution for each scenario, overlaid with the planned path.}

    \label{fig:scenarios}
\end{figure*}

\begin{figure*}[ht!]
    \centering
    \makebox[\textwidth][c]{
        \begin{subfigure}[b]{0.245\textwidth}
            \centering
            \includegraphics[width=\textwidth]{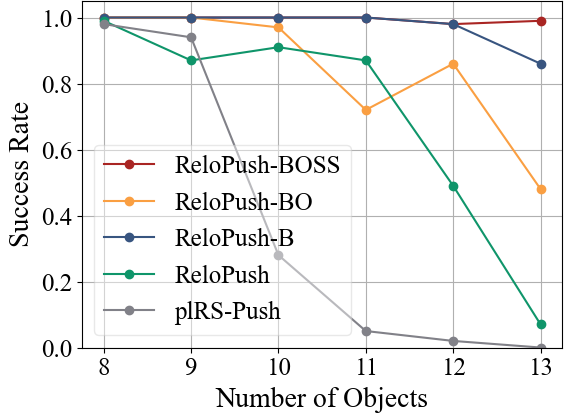}
            \caption{}
        \end{subfigure}
        \hfill
        \begin{subfigure}[b]{0.245\textwidth}
            \centering
            \includegraphics[width=\textwidth]{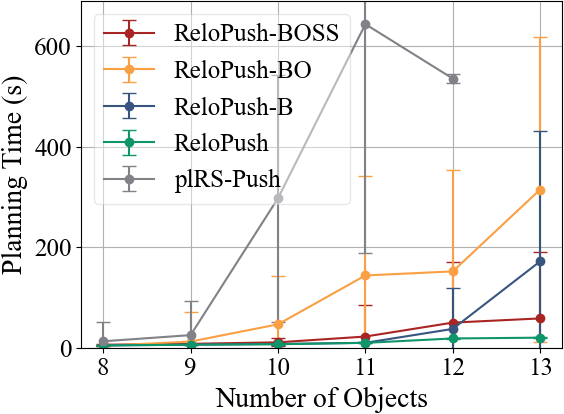}
            \caption{}
        \end{subfigure}
        \hfill
        \begin{subfigure}[b]{0.245\textwidth}
            \centering
            \includegraphics[width=\textwidth]{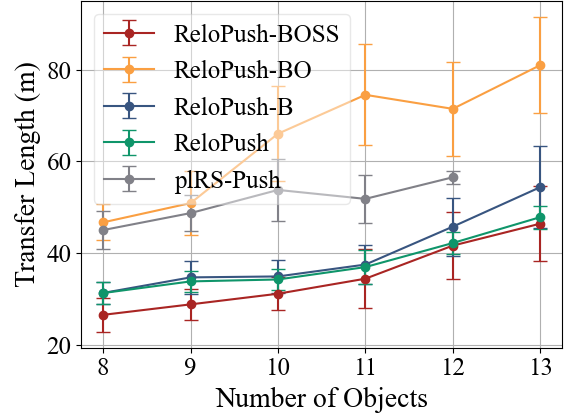}
            \caption{}
        \end{subfigure}
        \hfill
        \begin{subfigure}[b]{0.245\textwidth}
            \centering
            \includegraphics[width=\textwidth]{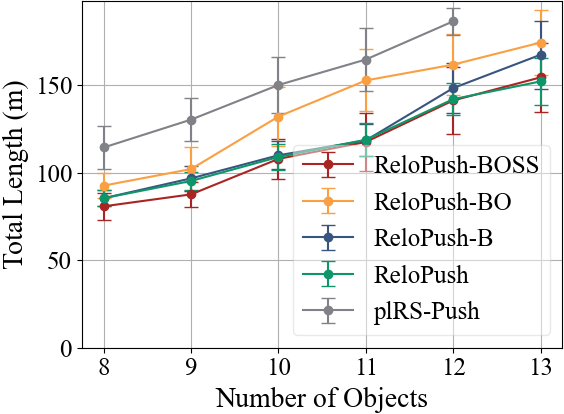}
            \caption{\label{fig:path-length}}
        \end{subfigure}

    }
        \vspace{-15px}
    \caption{Comparative planning performance across all metrics and scenarios.
    \label{fig:eval}}
\end{figure*}

\section{Evaluation}\label{sec:evaluation}

We present an empirical evaluation of \emph{ReloPush-BOSS}, assessing its scalability, robustness, and real-world transfer.







\subsection{Experiment design}

\textbf{Methods}. Our evaluation involves the following methods:

\begin{itemize}
    \item \emph{plRS-Push}. Adaptation of the piecewise-linear Rearrangement Search (plRS) \citep{krontiris2015dealing}, originally for pick-and-place, restricted to \emph{pushing} with a car-like robot. The high-level planner exhaustively enumerates object-order permutations (in randomized order) and returns a plan upon success or after all permutations are exhausted.

    \item \emph{ReloPush}. The vanilla version of ReloPush from our prior work~\citet{ahn2025relopush}.

    \item \emph{ReloPush-B}. \emph{ReloPush} augmented with a \emph{B}acktracking high-level planner (depth-first search) that exhaustively explores the solution tree (Sec.~\ref{sec:backtracking}). During graph construction, prerelocations are not optimized, and are found by sampling points along pushing poses~\citep{ahn2025relopush}.

    \item \emph{ReloPush-BO}. Replaces the search method for \emph{prerelocation} in \emph{ReloPush-B} with gradient-based \emph{O}ptimization but without informed warm-starting (Sec.~\ref{sec:optimization}).

    \item \emph{ReloPush-BOSS}. The proposed method, corresponding to \emph{ReloPush-BO}, augmented with a \emph{S}eeded-\emph{S}tart: each relocation optimization is warm started using the proposed warm-starting (Sec.~\ref{sec:optimization}).
\end{itemize}

\textbf{Metrics}. We evaluate performance in terms of the following metrics: \emph{Success rate} -- fraction of instances with a feasible plan found within 20 min (1{,}200 s); \emph{Planning Time (s)} -- time taken to find a feasible rearrangement plan; \emph{Total Pushing Length (m)} -- sum of all pushing segments for every object; \emph{Total Path Length (m)} -- overall distance traveled by the robot, including both transfer and transit segments. 

\textbf{Scenarios}. All methods were tested across the six scenarios of increasing hardness depicted in~\figref{fig:scenarios}, which involve the rearrangement of \(m=8\) to \(m=13\) objects in constrained workspace. To collect statistical insights, for each scenario, we generate 100 planning instances by locally perturbing the position and the orientation of each object or a goal within a range of $\pm 0.05 \text{ m}$ and $\pm 0.1 \text{ rad}$, respectively.





\subsection{Implementation}
For our experiments, we use an open-source 1/10 racecar from the MuSHR ecosystem~\citep{srinivasa2019mushr,talia2024hound}, augmented with a flat pushing bumper, following prior work~\citep{talia2023pushr,ahn2025relopush}. All scenarios are instantiated in a $4\times5.2\,\text{m}^2$ rectangular workspace (resolution $0.1\, \text{m}$ for map discretizations and collision checking). Objects are rigid square blocks with a $0.15\,\text{m}$ side, $0.44\,\text{kg}$ weight, and measured bumper–object friction of $\sim0.73$. Steering is constrained to a minimum turning radius of $\rho_p=1.43\,\text{m}$ to preserve quasistatic motion during pushing, and $\rho_{np}=1.09\,\text{m}$ during transit. Notably, $\rho_p$ constitutes $\sim35\%$ of the workspace width such that a full turning circle largely does not fit. Following limit-surface calculations~\citep{goyal1991planar,lynch1996stable,howcutkosky}, the minimum turning radius for stable pushing is $\rho_{p,min}=0.815\, \text{m}$, thus our setting represents a conservative bound. We exploit the symmetry of rearranged objects (cubic shape) to allow for multiple possible goal orientations (i.e., each object can be rearranged to one of four possible configurations). ReloPush-BOSS is implemented in C++. Dubins primitives and transit paths are computed using OMPL \citep{sucan2012open} and the Hybrid $A^*$ module of~\citet{wen2022cl}. Collision checking assumes exact geometry for the robot, objects, and workspace boundaries. We use Dijkstra's algorithm for searching on the PT-graph. We modified \textit{plRS} implementation by \citet{song2019object} into a \textit{plRS-Push} variant, suitable for pushing. Our code can be found at: \url{https://fluentrobotics.com/relopushboss}.


\subsection{Planning Performance}

Fig.~\ref{fig:eval} depicts the performance of ReloPush-BOSS against baselines. We see that ReloPush-BOSS attains the highest success rate across the board. Out of successful trials, ReloPush-BOSS also achieves the \emph{lowest} transfer length, and total length (see~\figref{fig:transfer_vis}), whereas its planning time is consistently among the lowest, close to ReloPush. 

\textbf{Hardness of benchmark}. Our benchmark is unique in that it involves the rearrangement of dense clutter of even up to 13 objects inside a constrained workspace. This results in a substantially denser workspace than prior work~\citep{ahn2025relopush,talia2023pushr,tang2023unwieldy, King-RSS-13}. This hardness can be traced in the path lengths required to execute the underlying scenarios, which reached larger than $150\, \text{m}$ for the more challenging ones (see~\figref{fig:path-length}).

\textbf{Optimization-guided robustness}. As the optimization module of ReloPush-BOSS searches over a wider range of $SE(2)$ for prerelocations, it results in a graph that makes more efficient use of the constrained space, resulting in higher success rates even as $m$ grows. A byproduct of that is that ReloPush-BOSS tends to find a solution earlier than other methods, leading to shorter planning times. This is nontrivial as a generic optimization approach generally takes longer than an analytic approach (see ReloPush-B).

\textbf{Value of initial guess}. We see that ReloPush-BO scales poorly with the scenario hardness, exhibiting lower success rates than ReloPush-BOSS and ReloPush-B, and planning inefficient rearrangement paths with excessively high times. In contrast, we see that ReloPush-BOSS, which is just ReloPush-BO augmented with our proposed module for warm-starting the optimizer, leads to scalable and robust performance. As discussed in Sec.~\ref{sec:optimization}, this gap in performance is explained by the nonlinearity of the domain which often traps ReloPush-BO into excessively high-cost local minima.


\textbf{Effectiveness of Depth-First Search}. From the results of ReloPush-B, we observe higher success rates and overall competitive performance against baselines just by integrating a depth-first search high level planner into ReloPush.


\begin{figure}[h!]
\centering
\begin{subfigure}{0.39\linewidth}
\includegraphics[width=\linewidth]{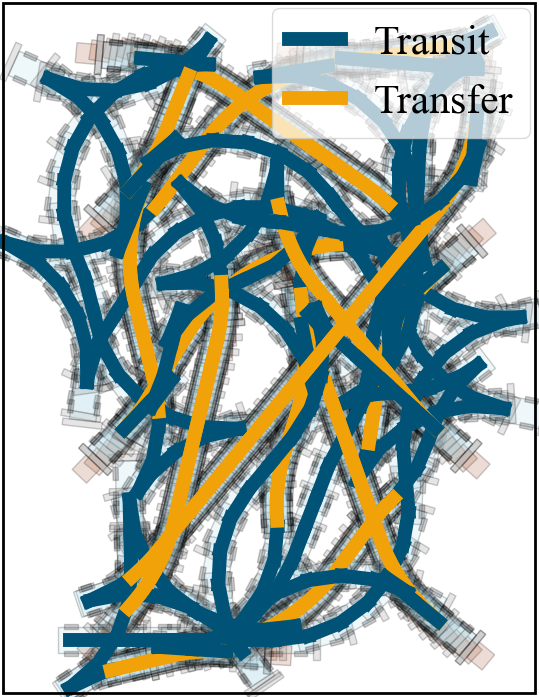}
\caption{ReloPush-BOSS}
\label{fig:ptgraph_path}
\end{subfigure}
\begin{subfigure}{0.39\linewidth}
\includegraphics[width=\linewidth]{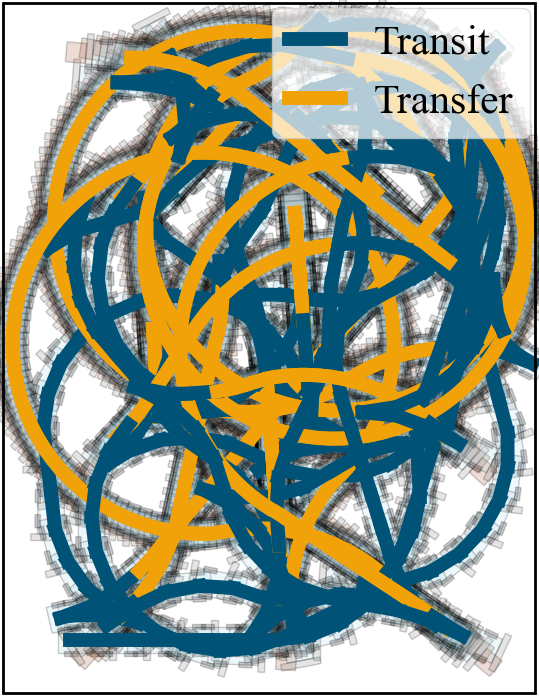}
\caption{plRS-Push}
\label{fig:ptgraph_}
\end{subfigure}
 \vspace{-1px}
    \caption{Paths resulting from tackling the $m=9$ instance using our method (\subref{fig:ptgraph_path}), and pIRS-Push (\subref{fig:ptgraph_}). Yellow and blue lines represent push-transfer and transit path portions, respectively.}
  \label{fig:transfer_vis}
   \vspace{-2pt}
\end{figure}

\vspace{-1pt}
\subsection{Real-Robot Demonstration}
We deployed the framework on a MuSHR platform~\citep{srinivasa2019mushr} to execute randomized benchmark instances. Using overhead motion capture for robot localization and MPC for path tracking, we successfully completed all trials, including a 13-object task, under open-loop object manipulation. These results (\figref{fig:scenarios}) demonstrate robustness against real-world friction and tracking errors. Footage from our experiments can be found at: \url{https://fluentrobotics.com/relopushboss}.


\vspace{-2px}
\section{Limitations}
The conservative bound on the turning radius makes planning in constrained spaces especially challenging. Ongoing work looks at relaxing it through the incorporation of data-driven models of pushing that account for variations in dynamics. Despite strong performance across complex scenarios, the proposed method is limited by its dependence on at least one unarranged object being accessible to the robot and by the consideration of at most one prerelocation per object. In severely cluttered workspaces, the robot may find no clear path to any object or optimization solution within the workspace. Future work will leverage impulsive manipulation to ``break'' the clutter and open up free paths. Finally, the implementation of obstacle clearing as a straight-line displacement is limiting and can be improved by leveraging the machinery of prerelocation optimization. 

\vspace{-2px}




\balance
\footnotesize
\bibliographystyle{abbrvnat}
\bibliography{references}



\end{document}